\documentclass[11pt,a4paper]{article}

\usepackage{fancyhdr}
\usepackage{amsmath,mathtools}
\usepackage{physics}
\usepackage{bm}
\usepackage{esint}
\usepackage{amsfonts}
\usepackage{amsthm}
\usepackage{amssymb}
\usepackage{amsbsy}
\usepackage{verbatim}
\usepackage{graphicx}
\usepackage{subfig}
\usepackage{url}
\usepackage{mathrsfs}
\usepackage[hmargin = 1in,vmargin=.75in]{geometry}
\usepackage{color}
\usepackage{amstext}
\usepackage{stmaryrd}
\usepackage{enumerate}
\usepackage{algpseudocode}
\usepackage{algorithm}
\usepackage{pifont}
\usepackage{prettyref}
\usepackage{hyperref}
\hypersetup{
    colorlinks = true,
}

\font\msbm=msbm10

\numberwithin{equation}{section}

\theoremstyle{plain}
\newtheorem{theorem}{Theorem}[section]
\newtheorem{lemma}[theorem]{Lemma}
\newtheorem{corollary}[theorem]{Corollary}

\newtheorem{assumption}[theorem]{Assumption}
\def\mathbb#1{\hbox{\msbm{#1}}}

\newcommand{\be}{\boldsymbol{e}}

\newcommand{\br}{\boldsymbol{r}}

\newcommand{\bu}{\boldsymbol{u}}
\newcommand{\bv}{\boldsymbol{v}}
\newcommand{\bw}{\boldsymbol{w}}
\newcommand{\bx}{\boldsymbol{x}}
\newcommand{\by}{\boldsymbol{y}}

\newcommand{\bone}{\boldsymbol{1}}

\newcommand{\BA}{\boldsymbol{A}}
\newcommand{\BB}{\boldsymbol{B}}
\newcommand{\BC}{\boldsymbol{C}}
\newcommand{\BD}{\boldsymbol{D}}

\newcommand{\BH}{\boldsymbol{H}}
\newcommand{\BI}{\boldsymbol{I}}
\newcommand{\BJ}{\boldsymbol{J}}

\newcommand{\BM}{\boldsymbol{M}}
\newcommand{\BN}{\boldsymbol{N}}

\newcommand{\BP}{\boldsymbol{P}}
\newcommand{\BQ}{\boldsymbol{Q}}

\newcommand{\BU}{\boldsymbol{U}}

\newcommand{\BX}{\boldsymbol{X}}
\newcommand{\BY}{\boldsymbol{Y}}
\newcommand{\BZ}{\boldsymbol{Z}}

\newcommand{\bpi}{\boldsymbol{\pi}}

\newcommand{\BDelta}{\boldsymbol{\Delta}}

\newcommand{\bphi}{\boldsymbol{\phi}}
\newcommand{\bpsi}{\boldsymbol{\psi}}
\newcommand{\bvphi}{\boldsymbol{\varphi}}

\newcommand{\BLambda}{\boldsymbol{\Lambda}}
\newcommand{\BSigma}{\boldsymbol{\Sigma}}

\newcommand{\CC}{\mathbb{C}}

\newcommand{\I}{\boldsymbol{I}}
\newcommand{\reals}{{\mathbb{R}}}

\newcommand{\lag}{\langle}
\newcommand{\rag}{\rangle}

\newcommand{\lp}{\left(} 
\newcommand{\rp}{\right)} 
\newcommand{\ls}{\left[} 
\newcommand{\rs}{\right]} 
\newcommand{\lc}{\left\{} 
\newcommand{\rc}{\right\}} 

\DeclareMathOperator{\Real}{Re}

\DeclareMathOperator{\Var}{Var}
\DeclareMathOperator{\mi}{\mathrm{i}}
\DeclareMathOperator{\E}{\mathbb{E}}

\DeclareMathOperator{\diag}{diag}

\DeclareMathOperator{\SNR}{SNR}
\DeclareMathOperator{\sign}{sign}

\DeclareMathOperator{\argmin}{argmin}

\renewcommand{\Pr}{\mathbb{P}}

\begin{document}
\title{\bf Improved theoretical guarantee for rank aggregation via spectral method}
\author{Ziliang Samuel Zhong\thanks{Shanghai Frontiers Science Center of Artificial Intelligence and Deep Learning, New York University Shanghai, Shanghai, China. S.L. and Z.S.Z. are (partially) financially supported by the National Key R\&D Program of China, Project Number 2021YFA1002800, National Natural Science Foundation of China (NSFC) No.12001372, Shanghai Municipal Education Commission (SMEC) via Grant 0920000112, and NYU Shanghai Boost Fund. Z.S.Z. is also supported by NYU Shanghai Ph.D. fellowship.} \thanks{Center for Data Science, New York University.}, Shuyang Ling$^*$}
 
\maketitle

\begin{abstract}

Given pairwise comparisons between multiple items, how to rank them so that the ranking matches the observations? This problem, known as rank aggregation, has found many applications in sports, recommendation systems, and other web applications. As it is generally NP-hard to find a global ranking that minimizes the mismatch (known as the Kemeny optimization), we focus on the Erd\"os-R\'enyi outliers (ERO) model for this ranking problem. Here, each pairwise comparison is a corrupted copy of the true score difference. We investigate spectral ranking algorithms that are based on unnormalized and normalized data matrices. The key is to understand their performance in recovering the underlying scores of each item from the observed data. This reduces to deriving an entry-wise perturbation error bound between the top eigenvectors of the unnormalized/normalized data matrix and its population counterpart.
By using the leave-one-out technique, we provide a sharper  $\ell_{\infty}$-norm perturbation bound of the eigenvectors and also derive an error bound on the maximum displacement for each item, with only $\Omega(n\log n)$ samples. Our theoretical analysis improves upon the state-of-the-art results in terms of sample complexity, and our numerical experiments confirm these theoretical findings.
\end{abstract}



\section{Introduction}

Given a subset of pairwise comparisons $H_{ij},~(i,j)\in E \subseteq [n]^2$ among $n$ items (or players) for some edge set $E$, how do we find their global ranking or scores $r_i$? This problem, known as rank aggregation, is ubiquitous across various areas, such as PageRank~\cite{BP98,G15}, recommendation systems~\cite{BL08,GL11}, and sports tournament~\cite{CVF13}. 
In practice, the pairwise comparison $H_{ij}$ typically comes in two forms: (a) cardinal: the pairwise score comparison $H_{ij} = r_i - r_j$ between two items is given (such as the outcomes of sport games); (b) ordinal: only which item is preferred by the voters is known, i.e., $H_{ij} = 1 $ if the item $i$ is preferred over the item $j$ and $H_{ij} = -1$ if otherwise. 

One natural way to find the global ranking of all the items that fit the observed data is to minimize the total mismatch:
\begin{equation}\label{def:kopt}
\min_{r_i} \sum_{(i,j)\in E} |\sign(r_i - r_j) - \sign(H_{ij})|
\end{equation}
where $E$ is the edge set consisting of all the pairs of $(i,j)$ on which a comparison is observed.
However, the optimization problem above, also known as the 
Kemeny optimization or the (weighted) feedback arc set problem,  is an NP-hard combinatorial optimization problem in general~\cite{A06,BM08,JLYY11}.

Instead of considering the general ranking problem, we assume the observed data are noisy pairwise measurements between the players' scores~\cite{DCT21,GL11}
\[
H_{ij} = 
\begin{cases}
r_i - r_j + \text{random noise}, & (i,j)\in E, \\
0, & (i,j)\notin E.
\end{cases}
\]
In particular, we focus on the Erd\"os-R\'enyi Outliers model (ERO)~\cite{C16,DCT21,GL11}: the underlying network $G=([n],E)$ is an Erd\"os-R\'enyi graph with probability $p$; and for $(i,j)\in E,$ the pairwise comparison $H_{ij}$ is either a clean pairwise offset $r_i -r_j$ with probability $\eta$ or a random outlier with probability $1-\eta$. 
Our goal is to recover the underlying scores $r_i$ via an efficient algorithm with theoretical guarantees. More precisely, we will investigate spectral methods and study how their performance depends on the network structure and noise strength.

\subsection{Related works}

Here we will provide a brief and non-exhaustive review of relevant literature on the statistical models of the ranking problem. Interested readers may refer to the works such as~\cite{CFMW19,C16,DCT21,SW17} and the references therein.
A general statistical model for ordinal measurements is in the form of $H_{ij} \sim \text{Ber}(p_{ij})$ where $p_{ij}$ depends on their underlying ranking such as parametric models~\cite{BT52,T27}, noisy orders (the Mallow's model)~\cite{BM09,M57}, noisy sorting model~\cite{BM08,MWR18}, and stochastic transitivity~\cite{SBGW16}. In particular, statistical ranking under parametric models has attracted much attention~\cite{SBGW16}. Famous examples include the Bradley-Terry-Luce (BTL) model~\cite{BT52} which considers
\[
H_{ij} \sim \text{Ber}(\theta_{ij}),~~\theta_{ij} = \frac{e^{r_i}}{e^{r_i} +e^{r_j}},
\]
and the Thurstone model~\cite{T27}, under which $H_{ij}\sim \text{Ber}(\Phi(r_i - r_j))$ where $\Phi(\cdot)$ is the CDF of the standard normal distribution. 

For cardinal measurements, it has been pointed out in~\cite{C16,DCT21}, the ranking problem is equivalent to the synchronization problem, i.e., to recover $r_i$ from their pairwise measurement $f(r_i,r_j)$ for some known function $f(\cdot)$. In particular, if the measurement is of the form $H_{ij} = r_i-r_j$, it is actually a synchronization problem on the additive group on $\mathbb{R}$ or also known as translation synchronization~\cite{AKT22,HLB17}. 
By mapping the measurement to the complex unit circle, the cardinal ranking problem can also be reformulated as phase synchronization~\cite{S11,C16}.

For the ranking with either ordinal or cardinal measurements, the existing approaches can be viewed as minimizing a surrogate of the Kemeny optimization~\eqref{def:kopt}. The ordinal measurements with parametric model fit nicely into the framework of generalized linear model with specific link function. As a result, maximum likelihood estimation is a natural candidate to recover the hidden scores from the binary measurement. 
The algorithm and performance of the MLE for the BTL models are studied in works such as~\cite{CGZ22,GSZ23,H04}. On the other hand, finding the MLE for noisy sorting problem is usually NP-hard~\cite{BM08,A06}. However, efficient algorithms are available to achieve near-optimal statistical performance~\cite{BM08,MWR18}.

Least squares method is among the most popular approaches in recovering global rankings~\cite{SBGW16,CNTL21,JLYY11,HKW10}.
By assuming a statistical model on the data generation process, the least squares methods have been studied and analyzed in~\cite{SBGW16} for stochastically transitive models and~\cite{CNTL21} for a general family of parametric models. The work~\cite{HLB17} proposed a truncated least squares approach to handle translation synchronization with outliers and analyzed the convergence of the iterative re-weighted least squares method under both deterministic and random noise models. The least squares has been shown to have rich connections with Hodge theory in~\cite{JLYY11}. The ranking problem also fits into the framework of low-rank matrix recovery: ~\cite{GL11,LVG21} wrote the ranking problem into a matrix completion problem and solved it via nuclear norm minimization and alternating minimization.

Spectral method~\cite{FDV16} is another class of powerful approaches. For ordinal measurements, a random-walk-based algorithm called rank centrality was proposed in~\cite{APA18,NOS17} to estimate the global ranking. A finite sample error bound between the estimation and the true scores was obtained in~\cite{NOS17}. Recently,~\cite{CFMW19} showed the spectral methods nearly match the maximum likelihood estimation for the data drawn from the BTL models. In~\cite{FDV16}, the authors constructed a similarity matrix based on the pairwise comparisons, and proposed a spectral seriation algorithm that computes the Fiedler eigenvector of the Laplacian with respect to the similarity matrix and then carried out the ranking task.

\vskip0.2cm

Our work is most relevant to~\cite{DCT21} by d'Aspremont, etc. In~\cite{DCT21}, the authors considered the statistical ranking under the Erd\"os-R\'enyi outlier  (ERO) model. An SVD-based algorithm was proposed to recover the relative ranking of the hidden scores. In~\cite{DCT21}, the authors studied two types of spectral methods that are based on the unnormalized and normalized data matrix, and then obtained $\ell_2$ and $\ell_{\infty}$ error bound on the singular vectors. While the  $\ell_2$-norm error bound is near-optimal, the entry-wise error between the eigenvectors from the data matrix and its population counterpart is sub-optimal in terms of the sample complexity. Our work bridges this gap between theory and practice by using the leave-one-out technique. This technique has been successfully used  to establish near-optimal $\ell_{\infty}$-perturbation bounds in a series of applications such as community detection~\cite{AFWZ20}, angular synchronization~\cite{ZB18,L22}, and ranking under the BTL model~\cite{CFMW19}.
Our contribution is mainly on providing an $\ell_\infty$-norm perturbation on the eigenvector of both unnormalized or normalized data matrix under the ERO model.
Our results indicate that only $\Omega(n\log n)$ pairs of comparisons are needed to provide a sharp entry-wise error bound. Moreover, by using the $\ell_{\infty}$-norm eigenvector perturbation, we also provide an error bound on the maximum displacement for each item, i.e., the total number of mismatched pairs for each item.  

\subsection{Outline of the paper}

The paper is organized as follows: we will introduce the Erd\"os-R\'enyi outliers model and spectral methods in Section~\ref{s:prelim}, and the main theorems are presented in Section~\ref{s:main}. Numerics are provided in Section~\ref{s:numerics} to complement our theoretical results; and the detailed proofs are deferred to Section~\ref{s:A}-\ref{s:C}.

\subsection{Notations}
We denote $\mi = \sqrt{-1}$ as the imaginary unit. Let $\BA \in \mathbb{C}^{n \times m}$ be a complex matrix and denote its $(i,j)$-entry by $A_{ij}$. We denote its transpose and  conjugate transpose as $\BA^{\top}$ and $\BA^H$ respectively.  The $\ell_2$-norm of a vector $\bv$ is denoted by $\norm{\bv} = \sqrt{\bv^H\bv} = \sqrt{\sum_{j=1}^n{|v_j|^2}}$ and its $\ell_\infty$ norm is denoted by $\norm{\bv}_{\infty} = \max_{1\le k \le n} |v_k|$, where  $v_k$ is $\bv$'s $k$-th entry. The inner product between two complex vectors $\bu$ and $\bv$ is defined as $\lag \bu,\bv\rag = \bu^H\bv$. For two vectors $\bu$ and $\bv$, we denote $\bu\propto \bv$ if they are parallel. 
We denote the operator 2-norm of $\BA$ as $\norm{\BA}$ which is the largest singular value of $\BA$. We denote the all-one vector in $\reals^n$ as $\bone_n$ and the all-one matrix in $\reals^{n \times n}$ as $\BJ_n$. 

Let $\BA$ and $\BB \in \mathbb{C}^{n \times m}$, and their Hadamard product is denoted as $\BA \circ \BB = \BC$ where $C_{ij} = A_{ij}B_{ij}$. We say $f(x) \gtrsim g(x)$ (or $f(x) = \Omega(g(x))$) and $f(x) \lesssim g(x)$ (or $f(x) = O(g(x))$) if there exists an absolute constant $C>0$ such that $f(x)\geq Cg(x)$ and $f(x) \leq Cg(x)$; $f(x)\asymp g(x)$ (or $f(x) = \Theta(g(x))$) if $c_1f(x)\leq g(x) \leq c_2f(x)$ for two absolute positive constants $c_1$ and $c_2$. We say $f(x) = \tilde{O}(g(x))$ if there exist absolute constants $C>0$ and $c_3 \in \mathbb{R}$ such that $f(x)\le Cg(x)\log^{c_3}(x)$. 

\section{Preliminaries}\label{s:prelim}

In this section, we will introduce the main statistical model, and spectral methods that are used to perform the ranking tasks. 

\subsection{Erd\"os-R\'enyi Outliers model (${\rm ERO}(\eta,p,n)$)}
\label{sec:problem setting}

We now introduce the Erd\"os-R\'enyi Outliers (or ${\rm ERO}(\eta,p,n)$) model. Let $\br = (r_1,\dots, r_n)^\top \in \reals^n$ be the unknown score vector whose $i$-th entry $r_i$ is associated with the $i$-the player, $1 \le i \le n$. The score value $r_i$ is assumed to be uniformly bounded: $r_i \in [-M,M]$, $1 \le i \le n$.
The pairwise score difference $H_{ij}$ between $i$-th and $j$-th players obeys the following statistical model:

\begin{equation}\label{def:H}
H_{ij} = 
\begin{cases}
r_i - r_j, & \text{with probability } \eta p, \\
Z_{ij} \overset{\text{i.i.d.}}{\sim} {\cal U}[-M,M], & \text{with probability }(1-\eta)p, \\
0, & \text{with probability } 1-p,
 \end{cases} ~\forall i < j
\end{equation}
and $H_{ij} = -H_{ji}$ where $Z_{ij}\sim{\cal U}[-M,M]$ stands for the uniform distribution over $[-M,M]$. It is easy to see that the probability of observing a pairwise measurement is $p$, i.e., $\Pr(H_{ij} \neq 0) = p$. Each observed pairwise comparison is either a clean score difference $r_i-r_j$ with probability $\eta$ or is corrupted by uniform random noise with probability $1-\eta$. In other words, $p \in (0,1]$ controls the proportion of observed pairwise measurements and  $\eta \in (0,1]$ is the corruption level.

To make the setting more clear, we write $H_{ij}$ into a unified form. Let $X_{ij}\sim$Bernoulli($p$) and $Y_{ij}\sim$Bernoulli($\eta$) be independent random variables. Then for each $i<j$, it holds
\begin{align*}
H_{ij} & = X_{ij}( Y_{ij} (r_i - r_j) + (1-Y_{ij})Z_{ij}) \\
& = \eta p(r_i - r_j) + (X_{ij} Y_{ij} - \eta p) (r_i - r_j) + X_{ij}(1-Y_{ij})Z_{ij}.
\end{align*}
By this definition, the data matrix $\BH$ is anti-symmetric, i.e., $\BH^\top = -\BH$ which satisfies
\begin{equation}\label{def:Hmat}
\BH = \eta p(\br\bone_n^{\top}- \bone_n\br^{\top})+ (\BX\circ \BY - \eta p \BJ_n)\circ (\br\bone_n^{\top}- \bone_n\br^{\top}) + \BX\circ(\BJ_n - \BY)\circ \BZ
\end{equation}
where $\BX$ and $\BY$ are symmetric and consist of $X_{ij}$ and $Y_{ij}$ respectively; the corruption matrix $\BZ$ is anti-symmetric.
To motivate the algorithm, it is more convenient to decompose the data matrix into $\BH = \bar{\BH} + \BDelta$ where 

\begin{equation}
\label{def:H_bar}
	\bar{\BH} = \E \ls \BH \rs =\eta p (\br\bone_n^{\top} - \bone_n\br^{\top})
\end{equation} 
is the rank-2 signal and 
\begin{equation}\label{def:Delta}
\BDelta = \BH - \bar{\BH} = (\BX\circ \BY - \eta p\BJ_n)\circ (\br\bone_n^{\top}- \bone_n\br^{\top}) + \BX\circ(\BJ_n - \BY)\circ \BZ
\end{equation}
 is the random noise.
The noise $\BDelta$ consists of the randomness from $X_{ij}$, $Y_{ij},$ and $Z_{ij}.$
The rank-2 signal $\bar{\BH}$ and the noise $\BDelta$ are both anti-symmetric. 

In this work, we will try to (a) answer how to recover the rank $r_1,\ldots,r_n$ from the noisy pairwise measurements $\BH$; (b) provide an error bound between the true and estimated ranking.

\subsection{Spectral ranking algorithm}
\label{s:Clgo}

We will introduce the two types of ranking algorithms based on the eigenvectors of the data matrix $\mi\BH$ and its normalized version.

\paragraph{Unnormalized spectral ranking}
Our goal is to recover the underlying ranking score $\br$ from the noisy pairwise measurements $\BH$. Before introducing the algorithm, we first note that multiplying an anti-symmetric matrix by the imaginary unit $\mi = \sqrt{-1}$ gives a Hermitian matrix with real eigenvalues and mutually orthogonal eigenvectors. Therefore, it suffices to consider $\mi\BH$ in our algorithm and analysis. To motivate the spectral method and see why it works, let's first consider the spectral decomposition of $\mi\bar{\BH}$. Given the SVD of $\bar{\BH}$:
\begin{align}
    \bar{\BH} = \eta p (\br \bone_n - \bone_n \br^\top)= \bar{\sigma} \lp \bar{\bu}_1 \bar{\bu}_2^{\top} - \bar{\bu}_2 \bar{\bu}_1^{\top} \rp,
\end{align}	
where
\[
\bar{\sigma} = \eta p\sqrt{n}\|\br - \alpha\bone_n\|, \qquad \bar{\bu}_1 =  -\frac{\bone_n}{\sqrt{n}}, \qquad \bar{\bu}_2 = \frac{\br-\alpha\bone_n}{\|\br-\alpha\bone_n\|}, \qquad\alpha = \frac{1}{n}\br^{\top}\bone_n,
\]
we know that the eigen-pairs of $\mi \bar{\BH}$ are given by
\[
\begin{aligned}
\bar{\bphi}_1 & = \frac{1}{\sqrt{2}}\lp\bar{\bu}_2 + \mi\bar{\bu}_1\rp, ~~\text{with eigenvalue }~\bar{\sigma},  \\
\bar{\bphi}_2 & = \frac{1}{\sqrt{2}}\lp\bar{\bu}_2 - \mi\bar{\bu}_1\rp, ~~\text{with eigenvalue }~ -\bar{\sigma},
\end{aligned}
\] 
or equivalently $\mi\bar{\BH} = \bar{\sigma}(\bar{\bphi}_1\bar{\bphi}_1^H - \bar{\bphi}_2\bar{\bphi}_2^H)$.
Note that the real part of $\bar{\bphi}_1$ is a 
centered and shifted version of the unknown ranking scores $\br$ and the imaginary part is uninformative. 

This motivates us to use the real part of the top eigenvector $\bphi_1$ (w.r.t. eigenvalue $\sigma$) of $\mi\BH$ to perform the ranking. Before doing that, we also need to resolve some ambiguities.
It is easy to see that $e^{i\theta}\bphi_1 $ is an eigenvector corresponding to $\sigma$ for any $ \theta \in [0,2\pi)$. To resolve this rotation ambiguity, we will choose the eigenvector corresponding to $\sigma$ such that the real part of $e^{\mi\theta}\bphi_1$ is orthogonal to $\bone_n$: 
\begin{align*}
& \Re\lag e^{\mi\theta}\bphi_1,\bone_n\rag =  \lag\Re \bphi_1,\bone_n \rag \cos\theta -  \lag \Im \bphi_1, \bone_n\rag \sin \theta = 0, \\
& \Im\lag e^{\mi\theta}\bphi_1,\bone_n\rag = -\lag \Im (e^{\mi\theta}\bphi_1), \bone_n\rag = -(\lag \Re\bphi_1,\bone_n\rag\sin\theta +\lag \Im\bphi_1,\bone_n\rag\cos\theta ) \geq 0.
\end{align*}
Then the angle of rotation is given by  
\begin{align}
\label{eqn:unnormalized angle}
	\widehat{\theta} =
\begin{cases}
\arctan (\frac{\lag\Re \bphi_1,\bone_n \rag}{ \lag \Im \bphi_1, \bone_n\rag}), &  \lag \Im \bphi_1, \bone_n\rag < 0, \\
-\frac{\pi}{2}\sign(\lag\Re \bphi_1, \bone_n\rag), &  \lag \Im \bphi_1, \bone_n\rag = 0, \\
\pi + \arctan (\frac{\lag\Re \bphi_1,\bone_n \rag}{ \lag \Im \bphi_1, \bone_n\rag}), & \lag \Im \bphi_1, \bone_n\rag > 0.
\end{cases}
\end{align}
We will use the real part of $e^{\mi\widehat{\theta}}\bphi_1$ to perform the ranking. 
Motivated by the discussion above, we introduce the spectral method that is summarized in Algorithm \ref{algo:HER}.

\begin{algorithm}[h]
\caption{Unnormalized spectral ranking}
\label{algo:HER}
\begin{algorithmic}[1]
\State {\bf Input:} Measurement matrix $\BH$ constructed in Section~\ref{sec:problem setting}
\State {\bf Output:} Rank estimations $\widehat{\bpi}$
\State Compute an eigenvector $\bphi_1$ of $\mi \BH$ w.r.t. its largest eigenvalue $\sigma$
\State Find the angle of rotation $\widehat{\theta}$ in \eqref{eqn:unnormalized angle}
\State {\bf Rank recovery:} Obtain the estimated rank $\widehat{\bpi}$ of $\br$ based on the entries of $\bx = \Real(e^{\mi\widehat{\theta}}\bphi_1)$
\end{algorithmic}
\end{algorithm}

\paragraph{Normalized spectral ranking}

Note that the players with higher scores are more likely to influence the eigenvector. To mitigate this issue, we also consider the spectral method based on the normalized measurement matrix. 
Define the degree matrix $\BD$ of the measurements $\BH$ as
\begin{equation}\label{def:D}
\BD := \diag(|\BH|\bone_n)=  \diag\left(\sum_{j=1}^n |H_{1j}|,\ldots, \sum_{j=1}^n |H_{nj}|\right)
\end{equation}
where $|\BH|$ takes the absolute value of each entry in $\BH.$
The left normalized measurement matrix is defined by
\begin{align}\label{eqn:definition of HL}
\BH_{\rm L} := \BD^{-1}\BH.
\end{align}
The normalized spectral ranking algorithm first computes the top eigenvector of $\BH_{\rm L}$ and then uses it to rank the items.

To see why it works, we consider a population version of $\BH_{\rm L}$, defined by 
\begin{equation}\label{def:HLbar}
\bar{\BH}_{\rm L} := \bar{\BD}^{-1}\bar{\BH} ~~\text{ where } ~~\bar{\BD} = \diag(\E [|\BH|] \bone_n).
\end{equation}
Note that $\mi\bar{\BH}_{L}$ is rank-2 with the normalized top eigen-pair (see Section~\ref{s:B} for details):
\[
\bar{\bpsi}_1 \propto \frac{\bar{\BD}^{-1}(\br - \gamma\bone_n)}{\norm{\bar{\BD}^{-1/2}(\br - \gamma\bone_n)}} - \mi\frac{\bar{\BD}^{-1}\bone_n}{\norm{\bar{\BD}^{-1/2}\bone_n}},~~~\bar{\xi} = \eta p \sqrt{n}\norm{\bar{\BD}^{-1/2}(\br - \gamma \bone_n)}\norm{\bar{\BD}^{-1/2} \bone_n}
\] 
where $\gamma = (\bone_n^\top \bar{\BD}^{-1}\bone_n)^{-1}\br^\top \bar{\BD}^{-1}\bone_n$ which ensures $\lag \Re\bar{\bpsi}_1, \bone_n\rag = 0.$
The real part of $\bar{\BD}\bar{\bpsi}_1$ is a rescaled and shifted version of $\br$ and the imaginary part is parallel to $\bone_n$.  

As a result, one can use $\BD\Re\bpsi_1$ to do the ranking where 
$(\xi,\bpsi_1)$ is the top eigen-pair of $\mi \BH_{\rm L}$. Similar to the unnormalized case, $\bpsi_1$ is unique modulo a global phase factor. In the algorithm, we conduct the ranking based on $\BD\Re(e^{\mi\theta} \bpsi_1)$ where $\theta$ is chosen in the same way as the unnormalized algorithm:
\[
\lag \operatorname{Re}(e^{\mi\theta} \bpsi_1), \bone_n \rag =\lag  \operatorname{Re}(e^{\mi\theta} \bpsi_1 ), \bone_n \rag = \lag \Re\bpsi_1, \bone_n\rag\cos\theta -  \lag \Im \bpsi_1, \bone_n\rag\sin\theta= 0
\]
and $\lag \operatorname{Im}(e^{\mi\theta}\bpsi_1 ), \bone_n \rag \geq 0. $
The phase $\theta$ is equal to
\begin{align}
\label{eqn:normalized angle}
	\widehat{\theta} =
\begin{cases}
\arctan (\frac{\lag\Re \bpsi_1, \bone_n \rag}{ \lag \Im \bpsi_1, \bone_n\rag}), &  \lag \Im \bpsi_1, \bone_n\rag < 0, \\
-\frac{\pi}{2} \sign(\lag\Re \bpsi_1, \bone_n \rag), &  \lag \Im \bpsi_1, \bone_n\rag  = 0,\\
\pi + \arctan (\frac{\lag\Re \bpsi_1, \bone_n \rag}{ \lag \Im \bpsi_1, \bone_n\rag}), & \lag \Im \bpsi_1, \bone_n\rag > 0.
\end{cases}
\end{align}
The normalized spectral ranking algorithm is summarized in Algorithm \ref{algo:NHER}.

\begin{algorithm}[h]
\caption{Normalized spectral ranking}
\label{algo:NHER}
\begin{algorithmic}[1]
\State {\bf Input:} Measurement matrix $\BH$ constructed in Section~\ref{sec:problem setting}
\State {\bf Output:} Rank estimations: $\widehat{\bpi}$
\State Compute an eigenvector $\bpsi_1$ of $\mi \BD^{-1}\BH$ w.r.t. its largest eigenvalue $\xi$
\State Find the angle of rotation $\widehat{\theta}$ in \eqref{eqn:normalized angle}
\State Obtain the estimated rank $\widehat{\bpi}$ of $\br$ based on the entries of $\BD\bx$ where $\bx = \Real(e^{\mi\widehat{\theta}}\bpsi_1)$
\end{algorithmic}
\end{algorithm}

\section{Main results}\label{s:main}

A natural question is how the performance of the spectral algorithm depends on the signal and noise strength.
A simple reason why this spectral method works is based on the eigenvector perturbation argument that if the noise $\norm{\BDelta}$ is small, then the top eigenvectors of $\mi\BH$ and $\mi\bar{\BH}$ are close. Therefore, the top eigenvector $\bphi_1$ provides an accurate estimation of $\bar{\bphi}_1$.  
Throughout our presentation, we will frequently refer to the following  assumption on the signal-to-noise ratio (SNR).
\begin{assumption}
The signal-to-noise ratio $(\SNR)$ is defined by
\begin{equation}\label{def:SNR}
     \SNR(\eta,p,n,\br,M) : = \sqrt{\frac{\eta^2 p n}{\log n}}\cdot \frac{\|\br - \alpha \bone_n\|}{\sqrt{n}M}\gtrsim 1,~~~\alpha = \frac{\lag \br, \bone\rag}{n}.
\end{equation}
For simplicity, we abbreviate $\SNR(\eta,p,n,\br,M)$ to $\SNR.$ 
\end{assumption}
Note that under $\SNR\gtrsim 1$, we automatically have $p \gtrsim n^{-1}\log n$ since $\|\br - \alpha\bone_n\| \lesssim \sqrt{n}M$ for $r_i\in[-M,M].$
Let's give a brief discussion of the $\SNR.$ Note that in Algorithm~\ref{algo:HER}, we compute the top eigenvector $\bphi_1$ of $\mi\BH$ to approximate the top eigenvector $\bar{\bphi}_1$ of its expectation $\mi\bar{\BH}$. It is a classical problem in matrix perturbation theory. The Davis-Kahan theorem~\cite{DK70} immediately gives the $\ell_2$-norm perturbation bound
\[
\min_{|\beta|=1} \norm{\bphi_1 -\beta \bar{\bphi}_1} \lesssim \frac{\norm{\BDelta}}{\bar{\sigma} - \norm{\BDelta}}.
\]
The bound is meaningful only when the right hand side is $O(1)$, that is, the signal is stronger than noise. We claim in Lemma \ref{lem:delta} that under the ERO model, the noise matrix $\BDelta$ satisfies $\|\BDelta\|\lesssim M\sqrt{pn\log n}$ with high probability. This motivates the definition of SNR in~\eqref{def:SNR}:
\[
\frac{\bar{\sigma}}{\|\BDelta\|} \approx \frac{\eta p \sqrt{n}\|\br - \alpha\bone_n\|}{M\sqrt{pn\log n}} =: \SNR \gtrsim 1.
\]  
A similar argument also applies to the normalized scenario. 

In other words, the assumption $\SNR\gtrsim 1$ ensures that the spectral algorithm provides a meaningful rank estimation under $\ell_2$-norm. It also ensures that the informative eigenvalues of $\mi \BH$ are well-separated. Let $\sigma = \sigma_1\geq \cdots \geq \sigma_n = -\sigma$  be the eigenvalues of $\mi\BH$. By Weyl's inequality (Theorem \ref{thm:weyl}), it holds that
\[
\bar{\sigma} - \norm{\BDelta} \le \sigma \le \bar{\sigma} + \norm{\BDelta},\qquad - \norm{\BDelta} \le \sigma_i \le \norm{\BDelta} \quad (i = 2,\ldots, n-1).
\]
This implies that under Assumption~\ref{def:SNR}, only the largest and smallest eigenvalues of $\mi \BH$ are significant and the others are located in $(- \norm{\BDelta}, \|\BDelta\|)$.

\subsection{$\ell_{\infty}$-norm perturbation}
\label{sec:main theorems}

Despite the Davis-Kahan bound usually provides a sharp bound under $\ell_2$-norm, it does not automatically yield a tight $\ell_{\infty}$-norm error bound between $\bphi_1$ and $\bar{\bphi}_1$ if we simply use $\ell_2$-norm to control the $\ell_{\infty}$-norm. On the other hand, $\ell_{\infty}$-norm perturbation bound is more useful in providing an error bound on the mismatch for each individual item. 

In the following theorems, we provide a sharp $\ell_\infty$-norm perturbation bounds for Algorithm~\ref{algo:HER} and~\ref{algo:NHER}, which only depends on the $\SNR$ defined in \eqref{def:SNR}. Define the relative $\ell_{\infty}$-error between $\bx$ and $\by$ by
\begin{equation}\label{def:S}
R(\bx,\by) := \min_{s\in\{\pm 1\}} \frac{ \norm{\bx /\|\bx\| - s\by/\|\by\|}_{\infty}}{\|\by\|_{\infty}/\|\by\|} = \min_{s\in \{\pm 1\}}  \frac{1}{\|\by\|_{\infty}}\left\| \frac{\|\by\|}{\|\bx\|} \bx - s\by \right\|_{\infty}
\end{equation}
The first theorem establishes an $\ell_{\infty}$-norm error bound between the top eigenvectors of $\mi\BH$ and its expectation $\mi\bar{\BH} = \mi\eta p (\br \bone_n^{\top} - \bone_n\br^{\top})$.

\begin{theorem}[\bf $\ell_\infty$-perturbation bound for Algorithm~\ref{algo:HER}]\label{thm:mainHER}
Let $\bphi_1$ and $\bar{\bphi}_1$ be the top eigenvector of $\mi\BH$ and $\mi\bar{\BH}$ respectively. Under Assumption~\ref{def:SNR}, it holds with high probability that
\begin{equation}\label{eq:HER_error}
\min_{|\beta|=1}\norm{\bphi_1 - \beta\bar{\bphi}_1}_{\infty} \lesssim \SNR^{-1}\norm{\bar{\bphi}_1}_{\infty}
\end{equation}
provided that $\SNR \gtrsim 1$ where $\beta = \lag \bar{\bphi}_1,\bphi_1\rag/|\lag \bar{\bphi}_1,\bphi_1\rag|.$ Moreover, 
\begin{equation}\label{eq:HER_error2}
R(\bx, \bar{\bx}) \lesssim \SNR^{-1},
\end{equation}
where $\bx = \Real(\bphi_1)$ and $\bar{\bx} = \Real(\bar{\bphi}_1).$
\end{theorem}
Now we proceed to present the main theorem for the normalized spectral ranking. Note that the normalized spectral ranking involves inverting $\BD$ in \eqref{def:D}. Therefore, for theoretical analysis, we impose a slightly stricter assumption.
Under the ERO model, direct computation gives
\begin{align*}
\bar{D}_{ii} & = \E D_{ii} = \sum_{j=1}^n \E |H_{ij}| = p\eta \sum_{j=1}^n |r_i - r_j| + \frac{p(1-\eta)(n-1)M}{2}.
\end{align*}
Define
\begin{equation}\label{def:lambda}
\lambda(\eta,p,n,\br,M) :=\frac{1}{pnM} \min_{1\leq i\leq n} \bar{D}_{ii} = \frac{1}{n} \left(\frac{\eta \min_{1\leq i\leq n}\|r_i \bone_n - \br\|_1}{M} + \frac{(1-\eta)(n-1)}{2} \right).
\end{equation}
We abbreviate $\lambda(\eta,p,n,\br,M)$ to $\lambda$ if no confusion arises.
Note that
\[
\lambda pnM \leq\bar{D}_{ii} \leq 2p\eta nM + \frac{p(1-\eta)(n-1)M}{2} \leq 2pnM
\]
and therefore, the condition number $\kappa$ of $\bar{\BD}$ satisfies
$\kappa(\bar{\BD}) \leq 2/\lambda.$

\begin{assumption}
For the $\SNR$ defined in~\eqref{def:SNR}, we assume:
\begin{equation}\label{def:SNR2}
\SNR \gtrsim \lambda^{-3} ~~\text{ implying  }~~\sqrt{\frac{pn}{\log n}} \gtrsim \lambda^{-3}.
\end{equation}
\end{assumption}
In fact, the assumption should be interpreted as $\SNR \gtrsim (2/\lambda)^3$ where $2/\lambda$ is an upper bound of the condition number. But for simplicity, we omit the constant. 
In particular, when $r_k = k$, $k=1,\cdots,n$ and $M=n$, it holds $\min_{1\leq i\leq n}\|r_i \bone_n - \br\|_1 \approx n^2/4 $ and then $\lambda\approx 1/4$. The next theorem for the normalized spectral ranking holds under~\eqref{def:SNR2}. 

\begin{theorem}[\bf $\ell_\infty$-perturbation bound for Algorithm~\ref{algo:NHER}]\label{thm:mainNHER}
Let $\bpsi_1$ and $\bar{\bpsi}_1$ be the top eigenvector of $\mi\BD^{-1}\BH$ and $\mi\bar{\BD}^{-1}\bar{\BH}$. Under Assumption~\ref{def:SNR2} it holds with high probability that 
\begin{equation}\label{eq:NHER_error}
\min_{|\beta|=1}\norm{\bpsi_1 - \beta\bar{\bpsi}_1}_{\infty} \lesssim \lambda^{-3}\SNR^{-1}\norm{\bar{\bpsi}_1}_{\infty}
\end{equation}
where $\beta = \lag \bar{\bpsi}_1,\bpsi_1\rag/|\lag \bar{\bpsi}_1,\bpsi_1\rag|.$ Moreover, it holds that
\begin{equation}\label{eq:NHER_error2}
R(\BD\bx, \bar{\BD}\bar{\bx})\lesssim \lambda^{-8} \SNR^{-1}
\end{equation}
where $\bx = \Real \bpsi_1$, $\bar{\bx} = \Real \bar{\bpsi}_1$, and $\bar{\BD}\bar{\bx}\propto \br - \gamma\bone_n.$
\end{theorem} 

To see why the two theorems above provide a sharp characterization of the perturbation, we consider a special case  when $r_k=k$ for $k=1,2,\ldots,n$. Then $\norm{\br - \alpha \bone_n} = \Theta(n^{3/2})$ and $\norm{\bar{\bphi}_1}_{\infty} = O(1/\sqrt{n})$.  Under Assumption~\ref{def:SNR}, $\bx = \Re\bphi_1$ satisfies
\[
\min_{s\in \{\pm 1\}}\| \bx - s \bar{\bx} \|_{\infty} \lesssim \SNR^{-1} \frac{\norm{\br-\alpha \bone_n}_{\infty}}{\norm{\br - \alpha \bone_n}} \lesssim \frac{\SNR^{-1}}{\sqrt{n}}
\]
with high probability where
\[
\SNR \approx \sqrt{\frac{\eta^2pn}{\log n}},\qquad~~\frac{\norm{\br-\alpha \bone_n}_{\infty}}{\norm{\br - \alpha \bone_n}} = O\left(\frac{1}{\sqrt{n}}\right).
\]

Figure \ref{fig:u2 plot uniform} visualizes the error bar for each entry of $\bx$ over $25$ experiments. We can observe that under both $\SNR =0.5$ and $0.8$, the error $\bx - \hat{s}\bar{\bx}$  distributes evenly across each entry where $\hat{s} = \argmin_{s\in\{\pm 1\}}\|\bx - s\bar{\bx}\|_{\infty}$. This empirically suggests that in this special example, the $\ell_{\infty}$ perturbation bound improves the naive $\ell_{\infty}$-error bound via $\ell_2$-bound by a dimension factor $n^{-1/2}$. 
\begin{figure}[h!]
\subfloat[$\SNR = 0.5$]{\includegraphics[width=0.42\columnwidth]{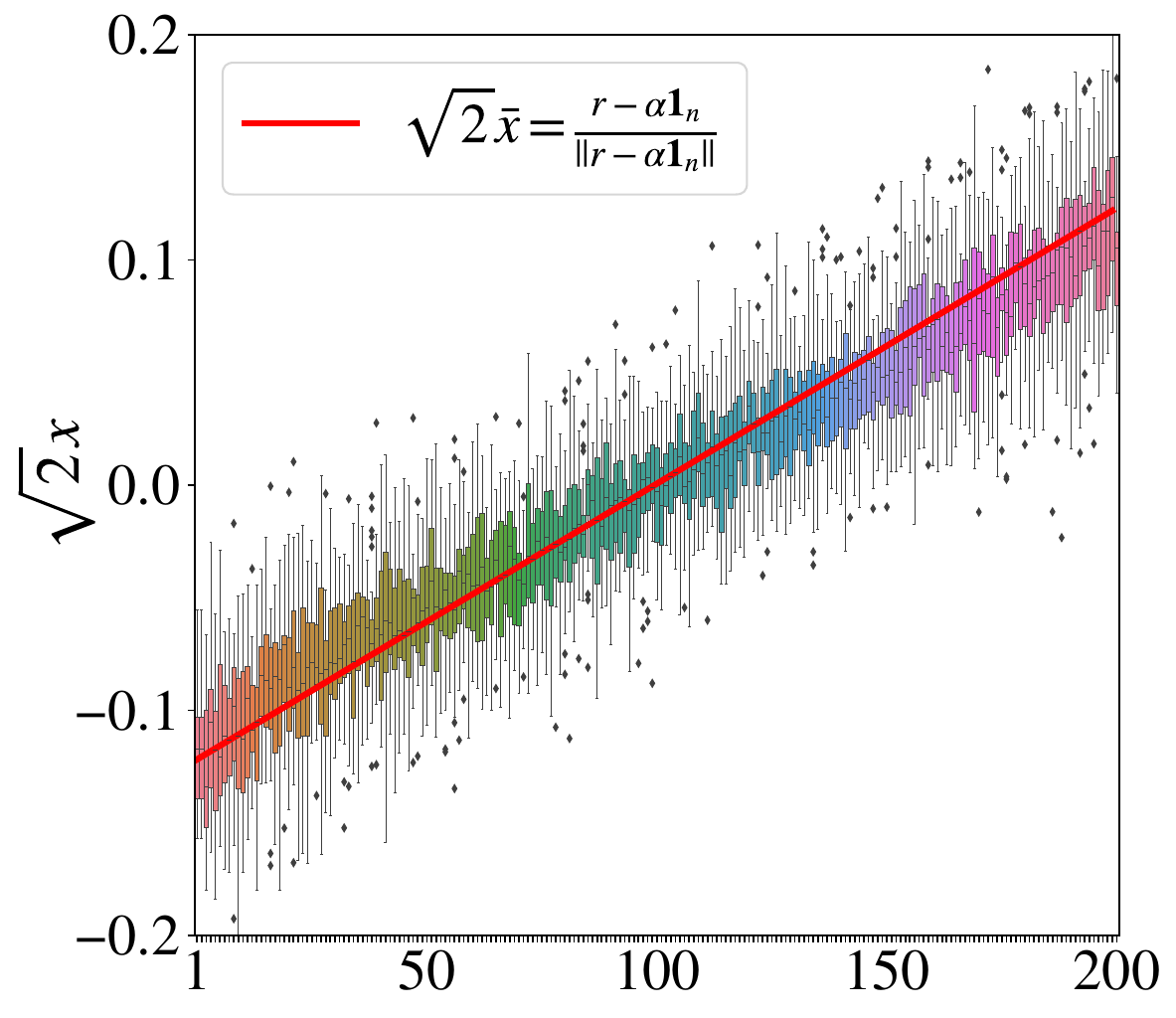}}
\hfill
\subfloat[$\SNR = 0.8$]{\includegraphics[width=0.42\columnwidth]{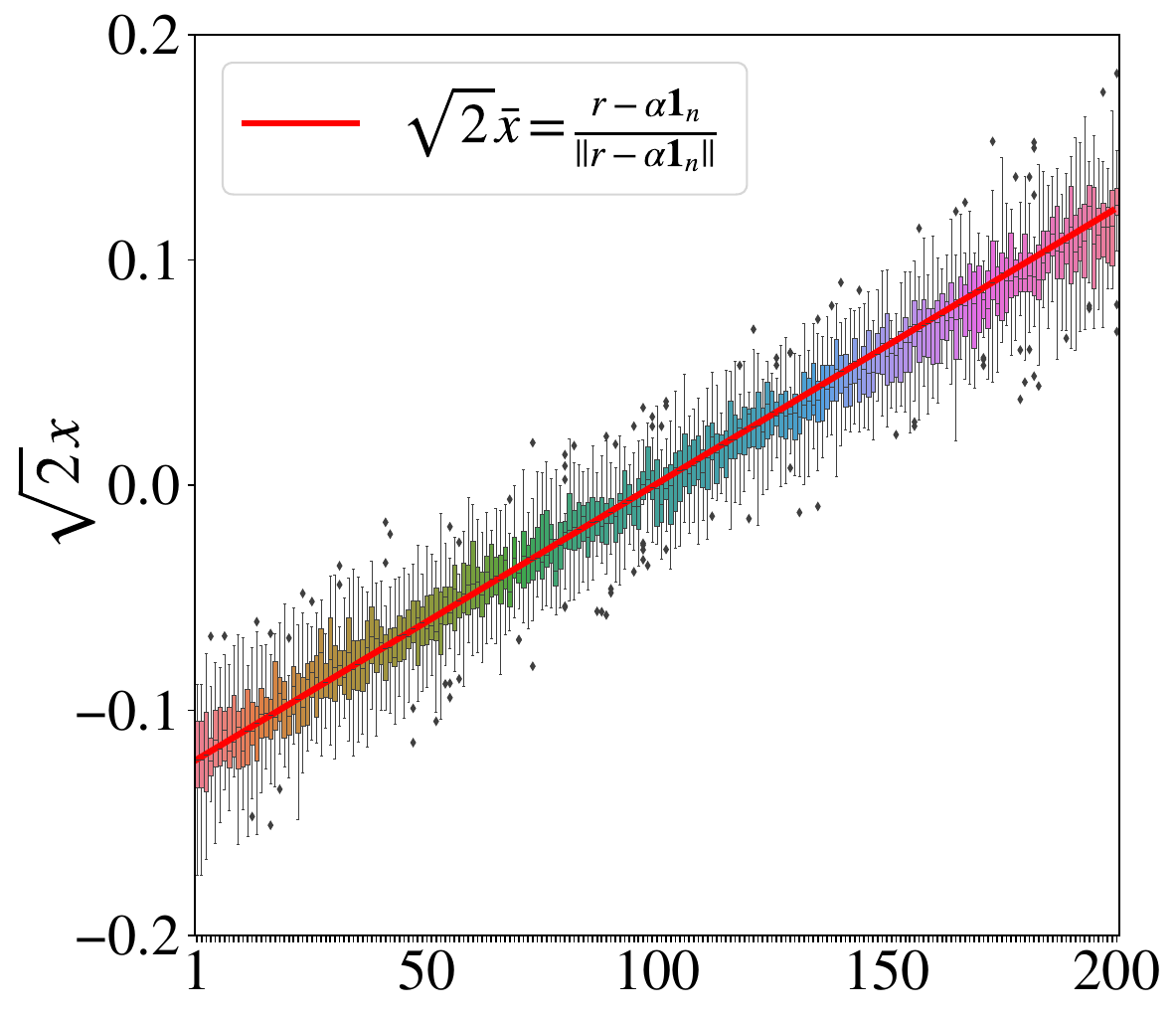}}
\caption{The error bar of $\bx$ and $\bar{\bx}$ (red line) for Algorithm~\ref{algo:HER} with ground-truth $r_k = k$, $1\leq k\leq 200$ with 25 experiments. The range of error bar  is even across each entry of $\bx - \hat{s}\bar{\bx}$. }
\label{fig:u2 plot uniform}
\end{figure}

Now we make a brief comparison between the state-of-the-art results and ours. One can verify that Algorithm~\ref{algo:HER} and~\ref{algo:NHER} are equivalent to the SVD-RS and SVD-NRS proposed in~\cite{DCT21} respectively. The authors of~\cite{DCT21} provide the $\ell_2$ perturbation bounds for both unnormalized and normalized algorithms and the $\ell_{\infty}$ perturbation bound only for the unnormalized algorithm. 
In particular, if $r_k =k$, their analysis guarantees that the unnormalized algorithm requires $\Omega(n^{4/3}\log^{3/2}n )$ observations to achieve $\ell_\infty$-error of order $\tilde{O}(n^{-1/2})$.

We provide $\ell_{\infty}$ perturbation bounds for both unnormalized and normalized spectral ranking algorithms.
Regarding the sample complexity, we note it suffices to ensure $\SNR \gtrsim 1$, i.e., $p \gtrsim \eta^{-2}n^{-1}\log n$, and then $p\binom{n}{2} \gtrsim \Omega(\eta^{-2}n\log n)$ pairwise measurements are needed to achieve $\ell_\infty$-norm error of order $O(n^{-1/2})$. 
Our results improve the state-of-the-art result in~\cite{DCT21} in terms of reducing sample complexity from $\Omega(n^{4/3}\log^{3/2}n)$ to $\Omega(n\log n)$.

\vskip0.25cm

Finally, we look into the distance between the recovered rank induced by $\bx$ and the ground truth rank induced by $\bar{\bx}$. We will first define metrics to measure the distance between permutations: let $\bpi_1$ and $\bpi_2$ be two arbitrary $n-$dimensional permutations, and then the displacement at index $i$ is defined by
\begin{align}
	\rho_i(\bpi_1,\bpi_2) := \frac{1}{n-1} \lp \sum_{j:\pi_1(j)>\pi_1(i)} \mathbf{1}_{\{\pi_2(j) <\pi_2(i)\}} + \sum_{j:\pi_1(j)<\pi_1(i)}\mathbf{1}_{\{\pi_2(j) >\pi_2(i)\}} \rp.
\end{align}
The displacement $\rho_i(\bpi_1,\bpi_2)$ counts the number of element pairs $(i,j)$ whose order under $\bpi_1$ is not preserved under $\bpi_2$ for each fixed $1\le i \le n$. The normalization factor $1/(n-1)$ makes $0\le \rho_i(\bpi_1,\bpi_2) \le 1$. Then we define the maximum displacement error $\rho_{\infty}(\bpi_1,\bpi_2)$ and the average displacement error $\bar{\rho}(\bpi_1,\bpi_2)$ as below
\begin{align}
\text{maximum displacement error:}\quad & \rho_{\infty}(\bpi_1,\bpi_2):= \max_{1\le i \le n} \rho_i(\bpi_1,\bpi_2), \label{eq:rhoinf}\\
\text{average displacement error:}\quad & \bar{\rho}(\bpi_1,\bpi_2):= \frac{1}{n}\sum_{i=1}^n \rho_i(\bpi_1,\bpi_2). \label{eq:rhobar}
\end{align}
The maximum displacement error $\rho_{\infty}(\bpi_1,\bpi_2)$ computes the maximum number of order violations over each entry. 
It will be used as a worst-case performance analysis of the algorithms and $\bar{\rho}(\bpi_1,\bpi_2)$ quantifies the average performance across all items.

For the ranking purpose, we assume that the entries of $r_i$'s are distinct, that is $r_i \neq r_j$ for any $1\le i< j\le n$. The following corollaries present how the $\ell_{\infty}$-norm perturbation bounds guarantee the maximum displacement error bounds $\rho_{\infty}(\bpi, \widehat{\bpi})$ where $\bpi$ is the rank induced by $\bar{\bx}$ and $\widehat{\bpi}$ is computed from $\bx$ for Algorithm~\ref{algo:HER} (from $\BD\bx$ for Algorithm~\ref{algo:NHER}). The proof follows exactly from Theorem 7 in~\cite{DCT21}, and we do not repeat it here.

\begin{corollary}[{\bf maximum displacement bound for Algorithm \ref{algo:HER}}]
    \label{cor:unnorm_disp}
    Conditioned on~\eqref{eq:HER_error2} in Theorem \ref{thm:mainHER}, the maximum displacement error is upper bounded by
    \[
     \rho_{\infty}(\bpi,\widehat{\bpi})\lesssim \frac{\norm{\br - \alpha \bone_n}}{n(\min_{i\neq j} |r_i-r_j|)}\SNR^{-1}\frac{\norm{\bar{\bx}}_{\infty}}{\|\bar{\bx}\|}.
    \]
\end{corollary}

\begin{corollary}[{\bf maximum displacement bound for Algorithm \ref{algo:NHER}}]
    \label{cor:norm_disp}
    Conditioned on~\eqref{eq:NHER_error2} in Theorem \ref{thm:mainNHER}, the maximum displacement error is upper bounded by
    \[
     \rho_{\infty}(\bpi,\widehat{\bpi})\lesssim \frac{\norm{\br - \gamma \bone_n}}{n(\min_{i\neq j} |r_i-r_j|)}\lambda^{-8}\SNR^{-1}\frac{\norm{\bar{\BD}\bar{\bx}}_{\infty}}{\norm{\bar{\BD}\bar{\bx}}}.
    \]
\end{corollary}
In particular, when we consider $r_k = k$, i.e., $\bpi = {\rm id}$ and let $\widehat{\bpi}$ be rank induced by $\bx$, then it holds
$\rho_{\infty}({\rm id},\widehat{\bpi}) \lesssim \SNR^{-1} $
where $\|\br - \alpha\bone_n\| = O(n^{3/2})$ and $\|\bar{\bx}\|_{\infty}/\|\bar{\bx}\| = O(1/\sqrt{n})$ for Algorithm~\ref{algo:HER} and $ \rho_{\infty}({\rm id},\widehat{\bpi})\lesssim \lambda^{-8} \SNR^{-1}$
where $\|\br - \gamma\bone_n\| = O(n^{3/2})$ and $\|\bar{\BD}\bar{\bx}\|_{\infty}/\| \bar{\BD}\bar{\bx}\| = O(1/\sqrt{n})$ for Algorithm~\ref{algo:NHER}.

\subsection{Main technique: the leave-one-out technique}\label{ss:tech}
We will introduce the main techniques and provide a sketch of proof of Theorem~\ref{thm:mainHER} and~\ref{thm:mainNHER}. The main idea follows from approximating the top eigenvector via one-step power approximation in~\cite{AFWZ20}. The approximation error can be well estimated via the leave-one-out technique. We will briefly discuss how this technique is implemented in each scenario. The detailed proofs of Theorem~\ref{thm:mainHER} and~\ref{thm:mainNHER} are deferred to Section~\ref{s:A} and~\ref{s:B} respectively.

\paragraph{Proof sketch for Theorem~\ref{thm:mainHER}}
Let $\beta = \lag \bar{\bphi}_1,\bphi_1\rag/|\lag \bar{\bphi}_1,\bphi_1\rag|.$ To bound $\norm{\bphi_1 - \beta\bar{\bphi}_1}_{\infty}$, the main idea is to find a surrogate $\widetilde{\bphi}_1$ that is close to $\bar{\bphi}_1$, and moreover $\norm{\bphi_1 - \beta\widetilde{\bphi}_1}_{\infty}$ is simple to estimate. For the unnormalized algorithm, we choose
\begin{equation}\label{def:tphi1}
\widetilde{\bphi}_1 = \frac{\mi\BH\bar{\bphi}}{\sigma}.
\end{equation}
Note that it is relatively simple to control the $\ell_{\infty}$-error between $\widetilde{\bphi}_1$ and $\bar{\bphi}_1$. Therefore, it suffices to show that $\|\bphi_1 - \beta\widetilde{\bphi}_1\|_{\infty}$ is rather small. 
The $\ell_\infty$ error between $\bphi_1$ and $\beta\widetilde{\bphi}_1$ satisfies
\begin{equation}
\begin{aligned}
 \norm{\bphi_1 - \beta\widetilde{\bphi}_1}_\infty  = \frac{1}{\sigma} \| \BH (\bphi_1 - \beta \bar{\bphi}_1) \|_{\infty} \le \underbrace{\frac{1}{\sigma} \norm{ \bar{\BH}(\bphi_1-\beta\bar{\bphi}_1)}_\infty}_{E_1}
+\underbrace{\frac{1}{\sigma} \norm{ \BDelta(\bphi_1-\beta\bar{\bphi}_1)}_\infty}_{E_2} 
\end{aligned}
\end{equation}
where $\mi\BH\bphi_1 = \sigma\bphi_1$ holds since $\bphi_1$ is the top eigenvector of $\mi\BH$ with eigenvalue $\sigma.$ The estimation of $E_1$ is straighforward while controlling $E_2 = \sigma^{-1} \norm{\mi \BDelta(\bphi_1-\beta\bar{\bphi}_1)}_\infty $ is the most technical  part because $\bphi_1-\beta\bar{\bphi}_1$ and each row of $\BDelta$ are statistically dependent. As a result, we cannot directly apply concentration inequalities to $E_2$ to have a tight bound. The remedy is to use the ``leave-one-out" technique to avoid the statistical dependence between $\BDelta$ and $\bphi_1-\beta\bar{\bphi}_1$. 

We introduce the following sequence of auxiliary matrices: for $1\leq k\leq n$, 
\begin{equation}\label{eq:loo}
\BH^{(k)} = \BH + \BDelta^{(k)} \qquad 
\Delta^{(k)}_{ij} 
= \begin{cases}
\Delta_{ij}, \quad i\neq k \text{ and }j\neq k, \\
0, \quad {\rm else}.
\end{cases}
\end{equation}
In other words, $\BDelta^{(k)}$ equals $\BDelta$ except the  $k$-th row and column. Let $\bphi^{(k)}_1$ be the top eigenvector of $\BH^{(k)}$ with eigenvalue $\sigma^{(k)}$. We will use $\bphi^{(k)}_1$ in the place of $\bphi_1$ and approximate $E_2$ by 
\begin{equation}\label{def:Tterm}
\begin{aligned}
E_2 &\leq \frac{1}{\sigma} \lvert \BDelta^\top_k(\bphi_1-\beta\bar{\bphi}_1)\rvert 
 = \frac{1}{\sigma} \lvert \BDelta^\top_k(\bphi_1-\beta^{(k)}\bphi^{(k)}_1 +\beta^{(k)}\bphi^{(k)}_1  - \beta\bar{\bphi}_1)\rvert \\
&  \le \underbrace{ \frac{1}{\sigma} \lvert \BDelta^\top_k(\bphi_1-\beta^{(k)}\bphi^{(k)}_1)\rvert}_{T_1}  
+ \underbrace{  \frac{1}{\sigma}\lvert \BDelta^\top_k(\beta^{(k)}\bphi^{(k)}_1  - \beta\bar{\bphi}_1)\rvert}_{T_2}
\end{aligned}
\end{equation}
where
\begin{equation}\label{def:beta_k}
\begin{aligned}
\beta^{(k)} &:= \argmin_{s \in \mathbb{C},|s|=1} \norm{\bphi_1-s\bphi^{(k)}_1 } = \frac{\lag \bphi_1^{(k)}, \bphi_1\rag}{|\lag \bphi_1^{(k)}, \bphi_1\rag|}.
\end{aligned}
\end{equation}

For $T_1$, Cauchy-Schwarz inequality implies that
\[
T_1\leq \sigma^{-1}\|\BDelta_k\| \|\bphi_1 - \beta^{(k)}\bphi_1^{(k)}\|.
\]
We will later show that $\norm{(\BDelta - \BDelta^{(k)})\bphi_1^{(k)}}$ is small, which controls $\|\bphi_1 - \beta^{(k)}\bphi_1^{(k)}\|$ by using the Davis-Kahan theorem.

For $T_2$, we take advantage of the independence among $\bar{\bphi}_1$, $\bphi^{(k)}_1 $ and $\BDelta_k$. By this statistical independency, using the matrix Bernstein inequality (Theorem \ref{thm:bernstein}) provides a sharp bound of $T_2$. The detailed estimations of $T_1$ and $T_2$ are deferred to Lemma~\ref{lem:Tterm}.

\paragraph{Proof sketch for Theorem~\ref{thm:mainNHER}}
The proof for the normalized algorithm is similar so we will only point out the differences in this section. The details are provided in Section~\ref{s:B}. We choose 
\[
\widetilde{\bpsi}_1 = \frac{\mi\BD^{-1}\BH\bar{\bpsi}_1}{\xi}
\] 
as an approximation to $\bpsi_1$. The key is to control the $\ell_{\infty}$ error between $\bpsi_1$ and $\widetilde{\bpsi}_1$.
Let $\beta = \lag \bar{\bpsi}_1,\bpsi\rag/| \lag \bar{\bpsi}_1,\bpsi\rag|,$ and then $\|\bpsi_1 - \beta\widetilde{\bpsi}_1\|_{\infty}$ satisfies
\begin{equation}\label{eq:psi_dec1}
\begin{aligned}
\|\bpsi_1 - \beta\widetilde{\bpsi}_1\|_{\infty} \le \frac{1}{\xi} \norm{ \mi \BD^{-1}\BH(\bpsi_1 - \beta\bar{\bpsi}_1) }_\infty 
\le \underbrace{\frac{1}{\xi} \norm{\BD^{-1}\bar{\BH}(\bpsi_1-\beta\bar{\bpsi}_1)}_\infty}_{E_1} + \underbrace{\frac{1}{\xi} \norm{ \BD^{-1}\BDelta(\bphi_1-\beta\bar{\bphi}_1)}_\infty}_{E_2} 
\end{aligned}
\end{equation} 
where $\BH = \bar{\BH} + \BDelta$ and $\BD^{-1}\BH\bpsi_1 = \xi\bpsi_1.$ We mainly focus on $E_2$:
\begin{align*}
E_2 & = \frac{1}{\xi} \norm{\BD^{-1}\BDelta(\bpsi_1-\beta\bar{\bpsi}_1)}_\infty \leq  \frac{1}{\xi d_{\min}} \max_{1\leq k\leq n} |\lag \BDelta_k, \bpsi_1 - \beta\bar{\bpsi}_1\rag|
\end{align*}
where $\BDelta_k$ is the $k$-th column of $\BDelta$ and $d_{\min} = \min_{1\le k \le n} D_{kk}$. Due to the statistical dependence between $\BDelta_k$ and $\bpsi_1 - \beta\bar{\bpsi}_1$, we introduce the same auxiliary sequence as \eqref{eq:loo}. Let $\bpsi^{(k)}_1$ be the top eigenvector of $\mi \bar{\BD}^{-1} \BH^{(k)}$ with eigenvalue $\xi^{(k)}$, i.e., $\mi \bar{\BD}^{-1} \BH^{(k)}\bpsi^{(k)}_1 = \xi^{(k)}\bpsi^{(k)}_1$. Then we will decompose $\xi^{-1}d_{\min}^{-1} | \BDelta_k^{\top}(\bpsi_1-\beta\bar{\bpsi}_1)|$ into two terms and find an upper bound of each one, i.e.,
\begin{equation}\label{def:Tterm_norm}
\begin{aligned}
\frac{1}{\xi d_{\min}} | \BDelta_k^{\top}(\bpsi_1-\beta\bar{\bpsi}_1)| 
& \lesssim \underbrace{\frac{1}{\xi d_{\min}}  | \BDelta_k^{\top}(\bpsi_1-\beta^{(k)}\bpsi_1^{(k)})|}_{T_1}  
+ \underbrace{\frac{1}{\xi d_{\min}}  | \BDelta_k^{\top}(\beta^{(k)}\bpsi_1^{(k)}-\beta\bar{\bpsi}_1)| }_{T_2}
\end{aligned}
\end{equation}
where $\beta^{(k)} := \argmin_{s \in \mathbb{C},|s|=1} \|\bpsi_1-s\bpsi^{(k)}_1\|$. The estimation of $T_1$ follows from Cauchy-Schwarz inequality and a variant of Davis-Kahan theorem (Theorem~\ref{thm:dk}); and $T_2$ uses the independence among $\bpsi_1^{(k)}$, $\bar{\bpsi}_1$, and $\BDelta_k$. The detailed estimation of $T_1$ and $T_2$ is provided in Lemma~\ref{lem:Tterm_norm}.

\section{Numerics}\label{s:numerics}
\subsection{Relative $\ell_{\infty}$ error/maximum displacement error v.s. ${\rm SNR}$}
\label{ss:l_inf}

We start with providing numerical evidence on the relative $\ell_{\infty}$-error v.s. $\SNR$ introduced in Theorem~\ref{thm:mainHER} and~\ref{thm:mainNHER}. 
We choose uniform $\br$ ($r_k = k, k=1,\dots,n$) as the ground truth. For each triplet $(\eta,n,p)$, we sample the data matrix $\BH$ from the ERO model, compute the top eigenvector $\bphi_1$ of $\BH$, and then calculate the average relative $\ell_{\infty}$-error $R(\bphi_1,\bar{\bphi}_1)$ in~\eqref{def:S} over 25 random instances. Figure~\ref{fig:l_inf_error} reports: (a) the average relative error v.s. varying SNR for different $200\leq n\leq1000$; (b) the average relative error v.s. $(p,\eta)$ for fixed $n=1000.$ Here for the uniform $\br$, we only present the results based on Algorithm~\ref{algo:HER} as it is highly similar to those from Algorithm~\ref{algo:NHER}. 

From Figure~\ref{fig:l_inf_error}, we can see that if the SNR is greater than 0.5, the relative error is roughly below 0.3; moreover, as the SNR increases, the relative error decreases. In particular, Figure~\ref{fig:l_inf_error}(b) shows the relative error on the red curve ($\SNR =0.5$), the green curve ($\SNR =0.8$) and the blue curve ($\SNR = 1.7$) are approximately $0.8$, $0.5$ and $0.2$ respectively. This confirms the relative error decays at the rate of $\SNR^{-1}$, as shown in Theorem~\ref{thm:mainHER}.

Figure~\ref{fig:l_inf_displace} shows the corresponding maximum displacement error by computing the estimated ranking $\widehat{\bpi}$ from $\bphi_1$ and the true ranking $\bpi = \text{id}$.
Figure~\ref{fig:l_inf_displace}(b) demonstrates that the maximum displacement error on the red curve ($\SNR =0.5$), the green curve ($\SNR =0.8$) and the blue curve ($\SNR = 1.7$) are approximately $0.4$, $0.3$ and $0.15$. This justifies the RHS of the error bound in Corollary~\ref{cor:unnorm_disp}.

\begin{figure}[h!]
\centerline{\includegraphics[width=0.9\columnwidth]{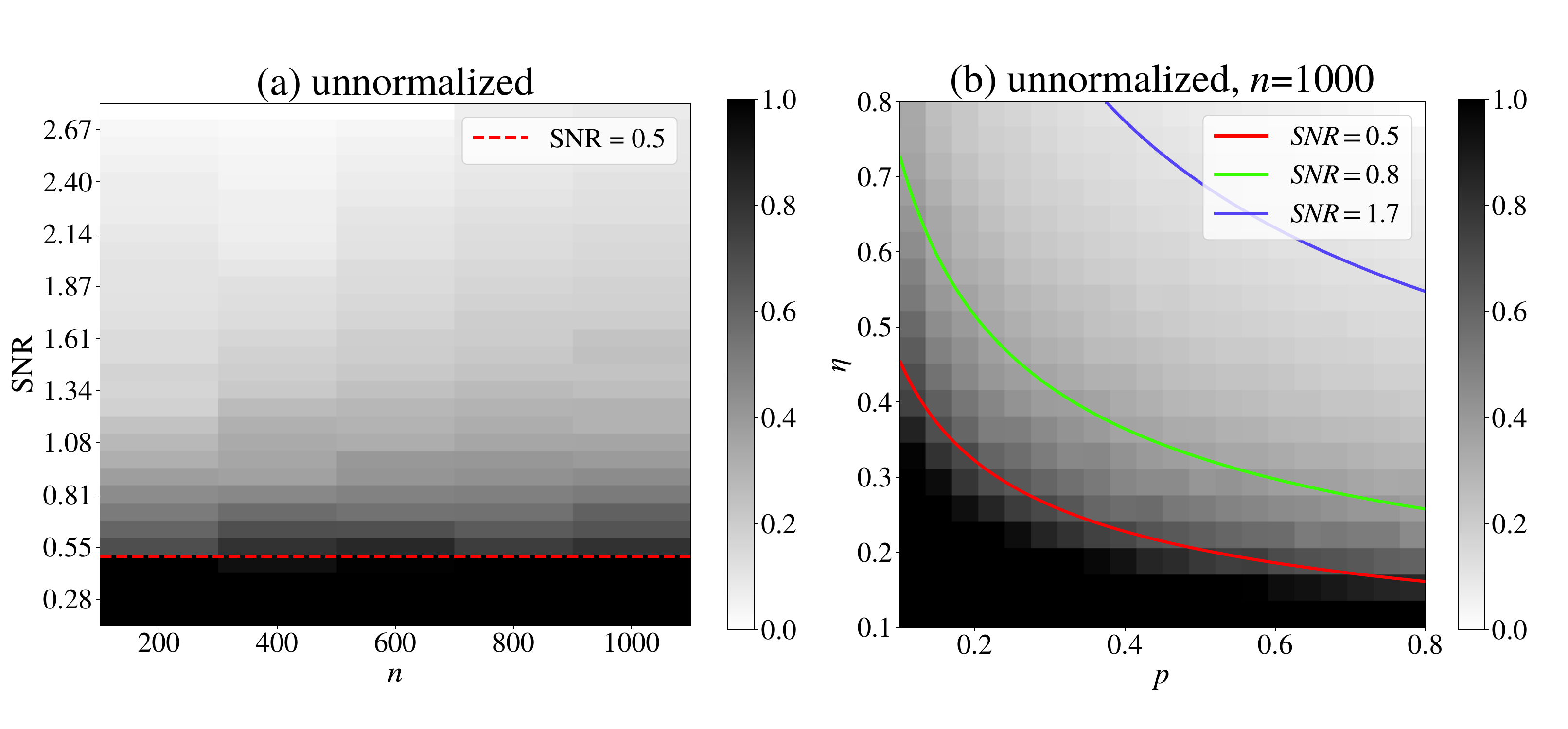}}
\caption{Relative $\ell_{\infty}$-error $R(\bx,\bar{\bx})$ for Algorithm~\ref{algo:HER} with ground-truth $r_k=k$, $k=1,\ldots,n$. Black region: error close to $1$; white region: error close to $0$.}
\label{fig:l_inf_error}
\end{figure}

\begin{figure}[h!]
\centerline{\includegraphics[width=0.9\columnwidth]{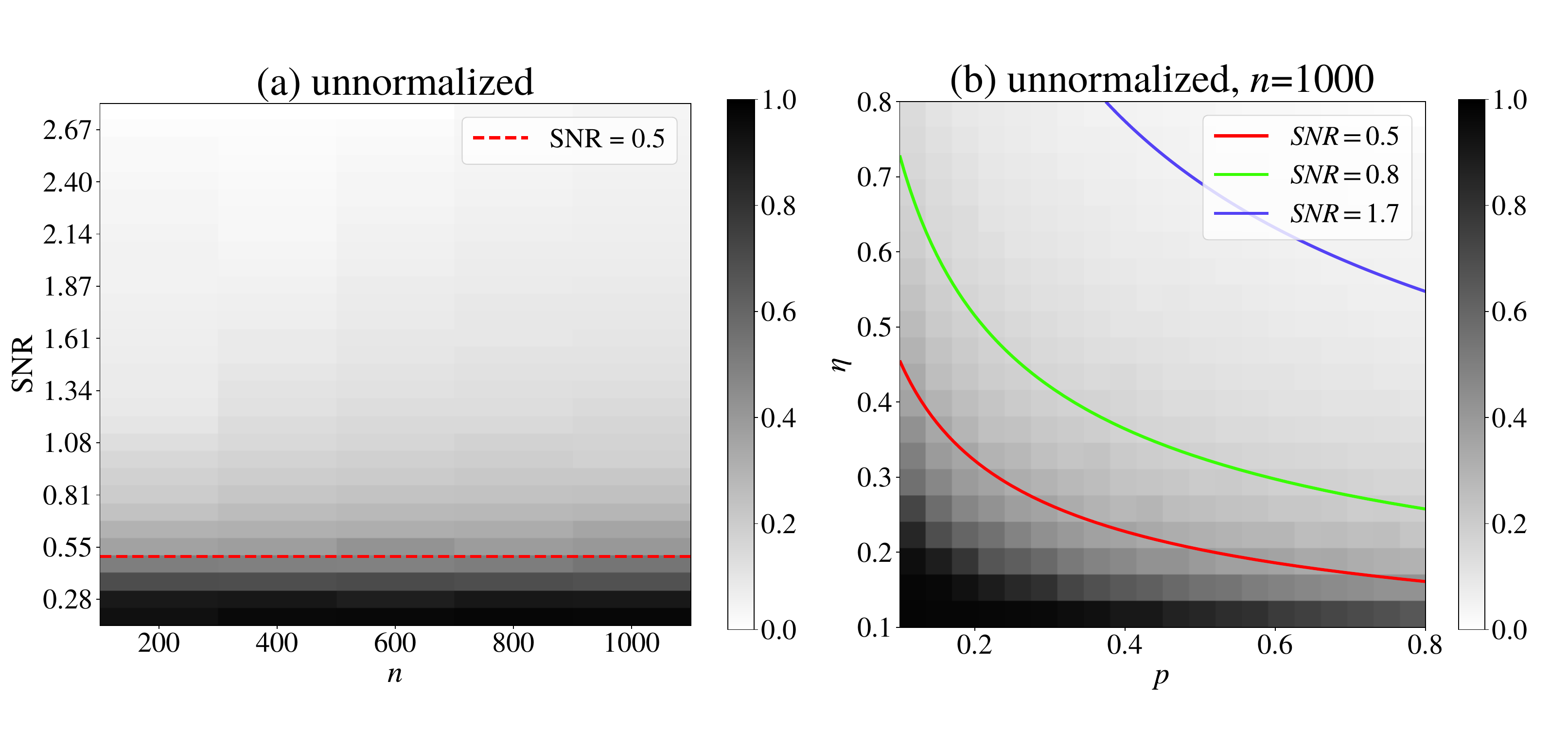}}
\caption{Maximum displacement error $\rho_{\infty}(\bpi,\widehat{\bpi})$ for Algorithm~\ref{algo:HER}  with ground-truth $r_k=k$, $k=1,\ldots,n$. Black region: error close to $1$; white region: error close to $0$.}
\label{fig:l_inf_displace}
\end{figure}

\subsection{Comparison between Algorithm~\ref{algo:HER} and~\ref{algo:NHER}}
\label{ss:algo_comp}

To study the difference between Algorithm~\ref{algo:HER} and~\ref{algo:NHER}, we will choose the sorted Gamma distributed $\br$ (each $r_k$ is sampled from Gamma distribution with parameters $a=b=1$ and $\br$ is sorted so that $\bpi = {\rm id}$) as the ground-truth.
This $\Gamma(1,1)$ distributed $\br$ makes the node degree skewed.
Here we compare two algorithms on $\br$ sampled from Gamma distribution.

 The settings of $\SNR$ and also $(\eta,n,p)$ are the same as those in Section~\ref{ss:l_inf}. We compute the average relative $\ell_{\infty}$-error $R(\bx,\bar{\bx})$ for Algorithm~\ref{algo:HER} and $R(\BD\bx,\bar{\BD}\bar{\bx})$ for Algorithm~\ref{algo:NHER} over 25 instances. 
In Figure~\ref{fig:l_inf_gamma}, we can see the main performance difference between two algorithms occurs when $p<0.2$. In this case, the measurement graph is not highly connected, which can lead to high variance in node degree. This issue can be mitigated by normalizing the data matrix via the degree so that all the items are of similar strengths.
For the corresponding maximum displacement error $\rho_{\infty}(\bpi,\widehat{\bpi})$, Figure~\ref{fig:maxdisplace_gamma} shows both algorithms perform similarly.

\begin{figure}[h!]
\begin{minipage}{0.45\textwidth}
\includegraphics[width=72mm]{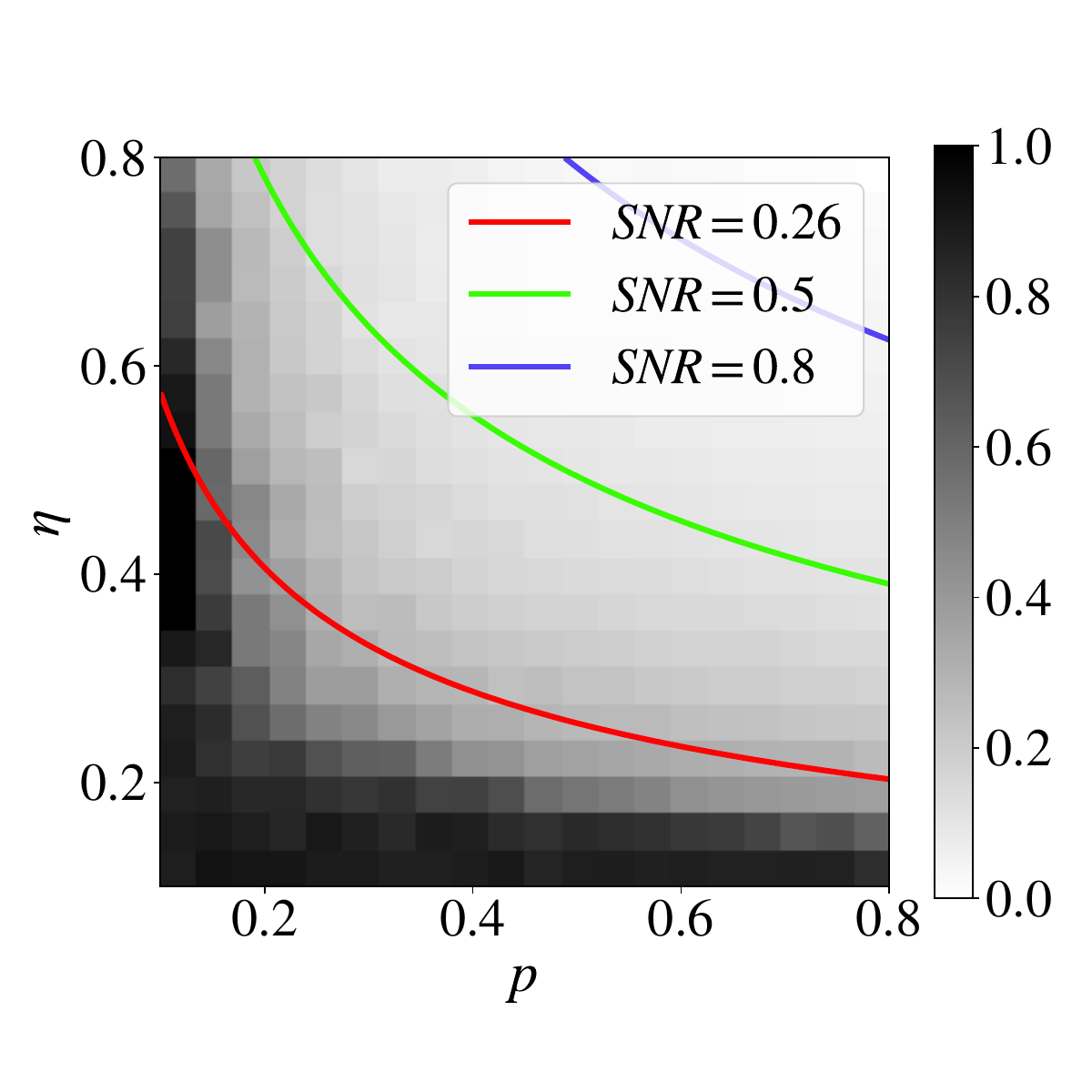}
\end{minipage}
\hfill
\begin{minipage}{0.45\textwidth}
\includegraphics[width=72mm]{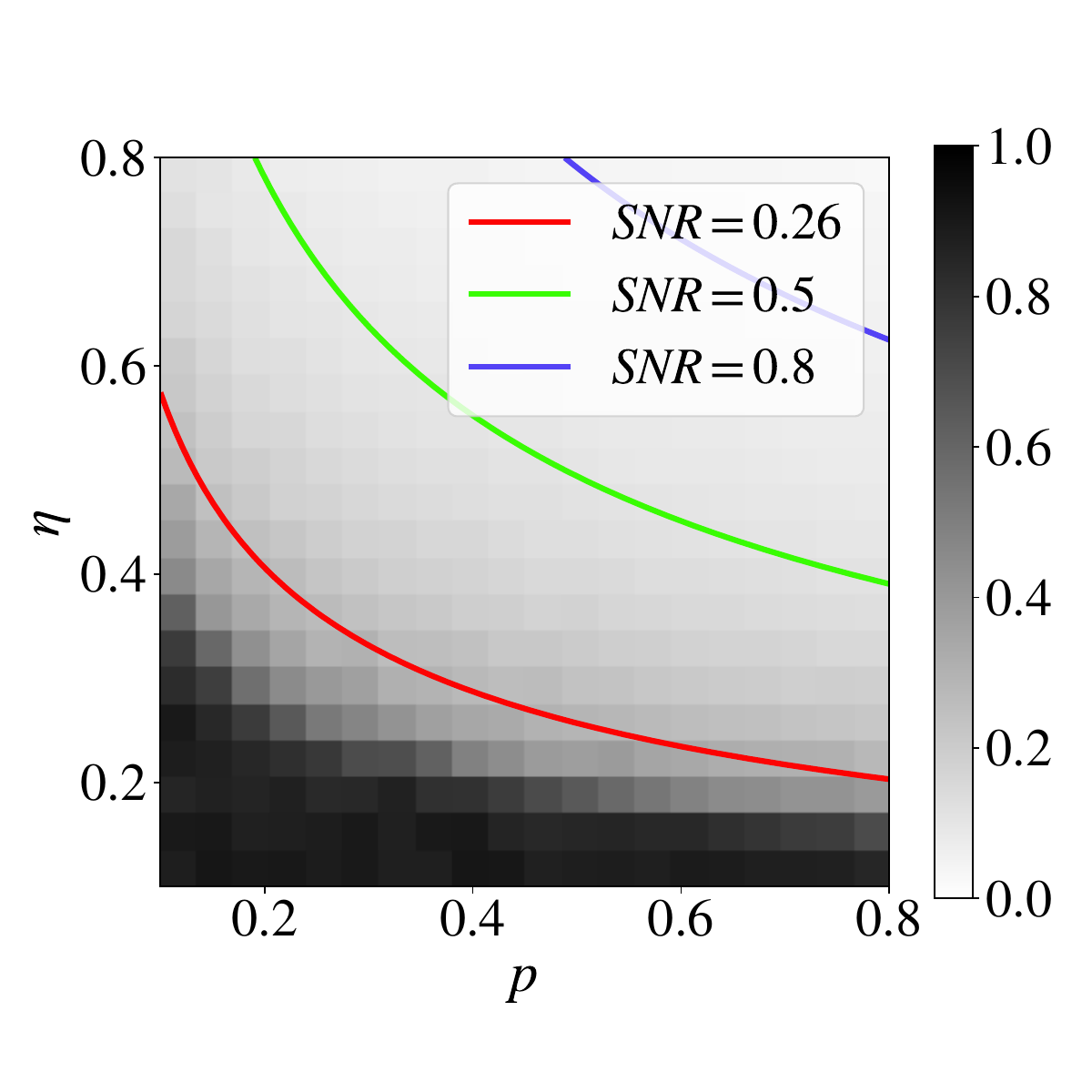}
\end{minipage}
\caption{Relative $\ell_{\infty}$-error with $\br$ sampled from Gamma distribution with $n=1000$. Left: $R(\bx,\bar{\bx})$ for Algorithm~\ref{algo:HER}; Right: $R(\BD\bx,\bar{\BD}\bar{\bx})$ for Algorithm~\ref{algo:NHER}}
\label{fig:l_inf_gamma}
\end{figure}

\begin{figure}[H]
\begin{minipage}{0.45\textwidth}
\includegraphics[width=72mm]{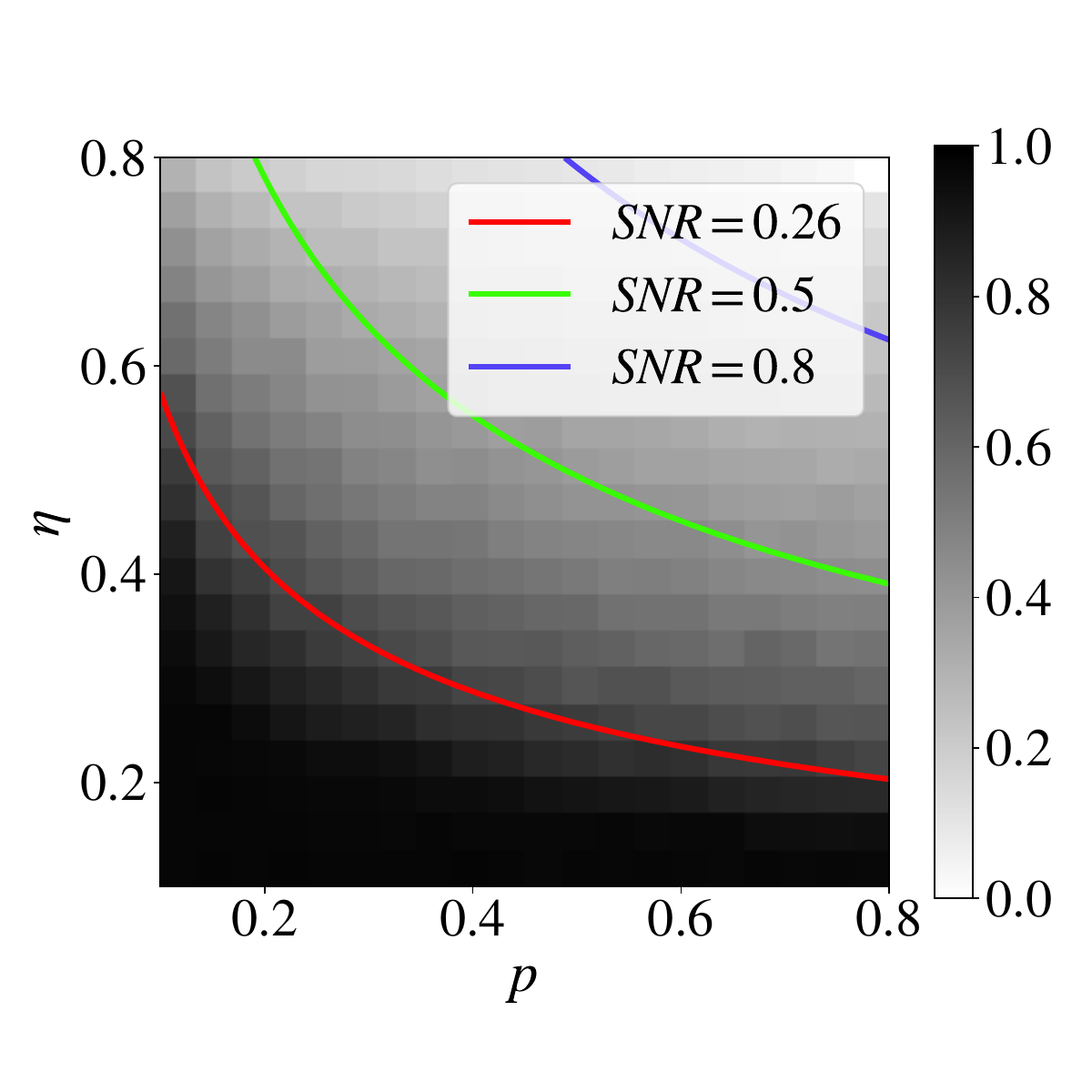}
\end{minipage}
\hfill
\begin{minipage}{0.45\textwidth}
\includegraphics[width=72mm]{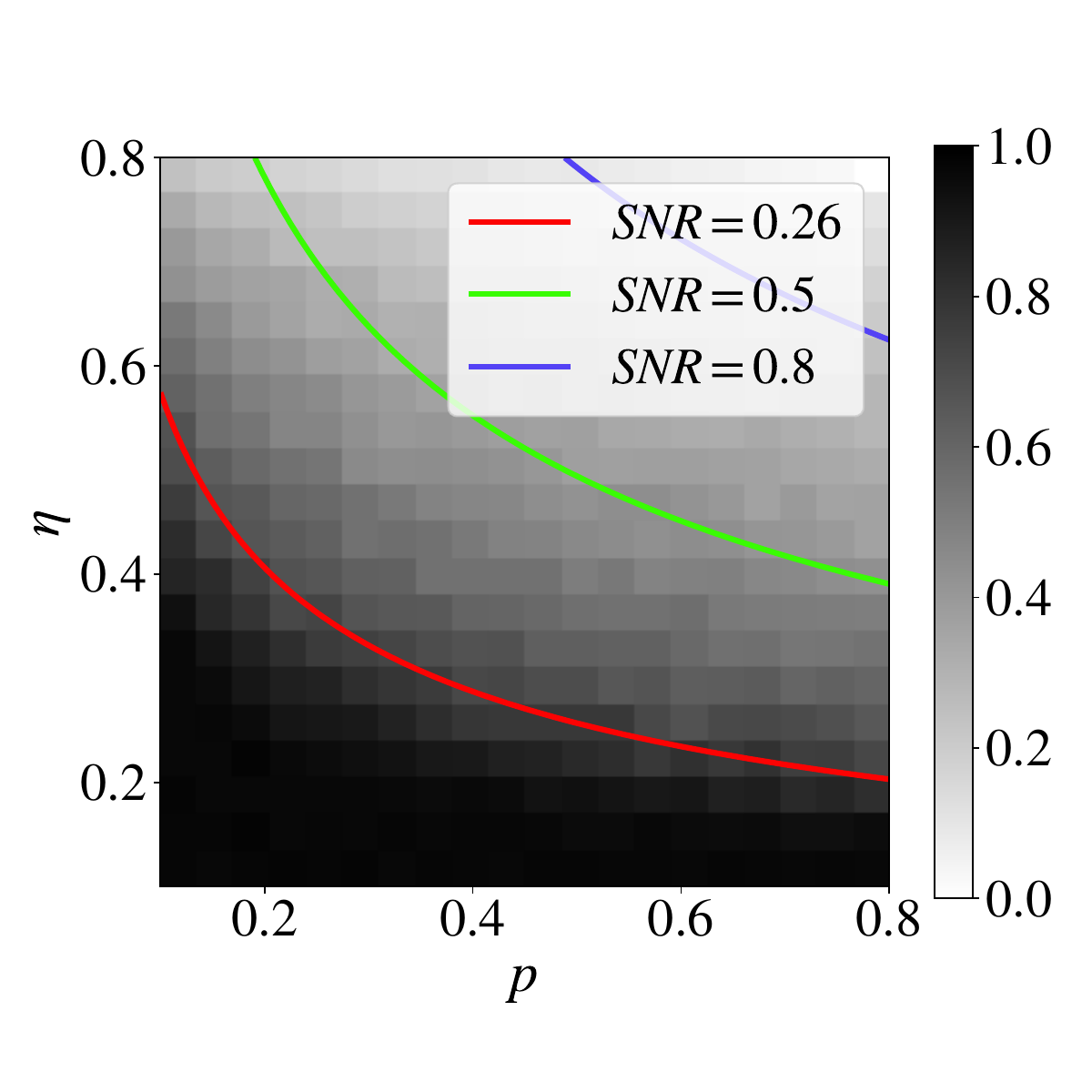}
\end{minipage}
\vfill
\begin{minipage}{0.45\textwidth}
\includegraphics[width=72mm]{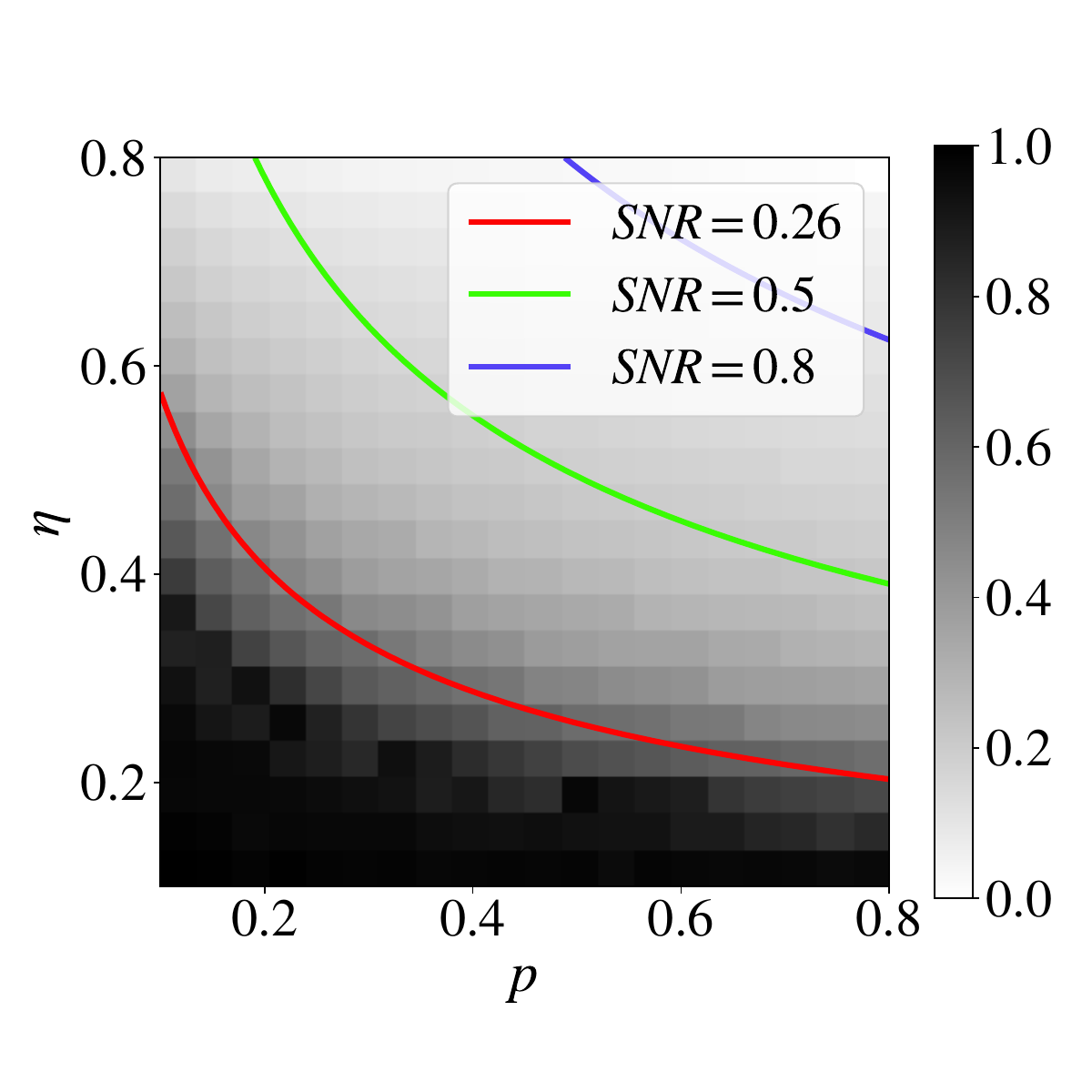}
\end{minipage}
\hfill
\begin{minipage}{0.45\textwidth}
\includegraphics[width=72mm]{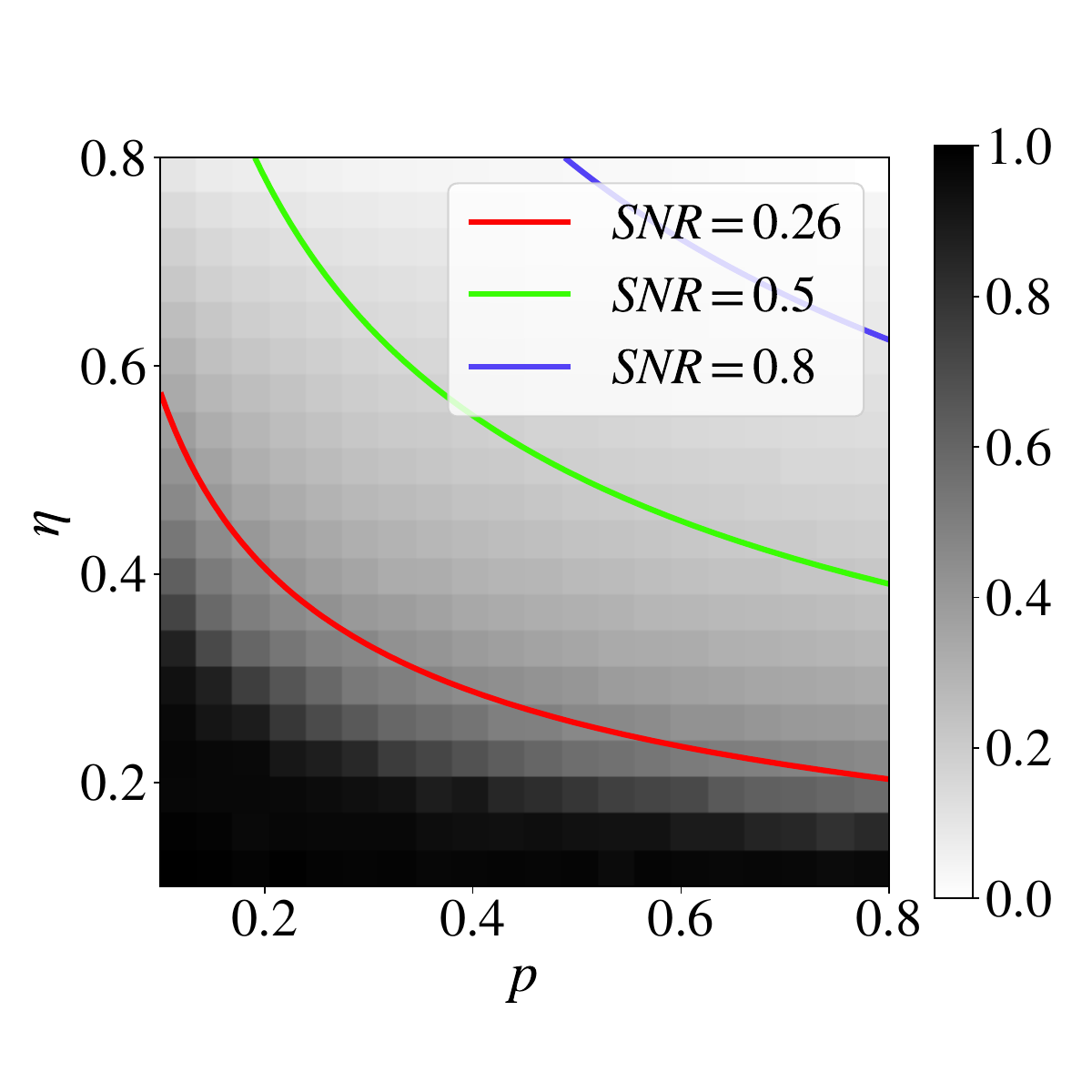}
\end{minipage}
\caption{Maximum displacement error ($\rho_{\infty}(\bpi,\widehat{\bpi})$, top row) and the average displacement error ($\bar{\rho}(\bpi,\widehat{\bpi})$, bottom row) for both algorithms with ground-truth sorted Gamma distributed $\br$. \\ Top left: $\rho_{\infty}(\bpi,\widehat{\bpi})$ for Algorithm~\ref{algo:HER}; top right: $\rho_{\infty}(\bpi,\widehat{\bpi})$ for Algorithm~\ref{algo:NHER}; bottom left: $\bar{\rho}(\bpi,\widehat{\bpi})$ for Algorithm~\ref{algo:HER}; bottom right: $\bar{\rho}(\bpi,\widehat{\bpi})$ for Algorithm~\ref{algo:NHER}}
\label{fig:maxdisplace_gamma}
\end{figure}

One interesting observation is: Figure~\ref{fig:l_inf_error}(b) and~\ref{fig:l_inf_gamma}(a) show the relative $\ell_{\infty}$-error $R(\bx,\bar{\bx})$ looks similar
under the uniform and Gamma distributed $\br$. However, the corresponding $\rho_{\infty}(\bpi, \widehat{\bpi})$ behaves very differently  under the same SNR: $\rho_{\infty}(\bpi, \widehat{\bpi})$ is much larger for the skewed distributed $\br$, as shown in Figure~\ref{fig:l_inf_displace}(b) and~\ref{fig:maxdisplace_gamma}(top left).
If we look into the average displacement error in Figure~\ref{fig:maxdisplace_gamma}, we can see $\bar{\rho}(\bpi,\widehat{\bpi})$ is much smaller than $\bar{\rho}_{\infty}(\bpi,\widehat{\bpi})$. This motivates us to understand why the average and maximum displacement error differ much. 

Corollary~\ref{cor:unnorm_disp} and~\ref{cor:norm_disp} imply the maximum displacement is proportional to the inverse of minimum separation between among the true ranking scores. For Gamma distribution, this separation can be small compared with the uniform scores.  Moreover, Figure~\ref{fig:displacement} indicates the error bars of $\rho_{i}(\bpi,\widehat{\bpi})$ (over 25 instances) are not even for each index $i$. This also explains why $\rho_{\infty}$ differs much from $\bar{\rho}$ for $\br$ sampled from Gamma distribution. 

\begin{figure}[H]
\subfloat[$\SNR =0.26$, unnormalized]{\includegraphics[width=0.45\columnwidth]{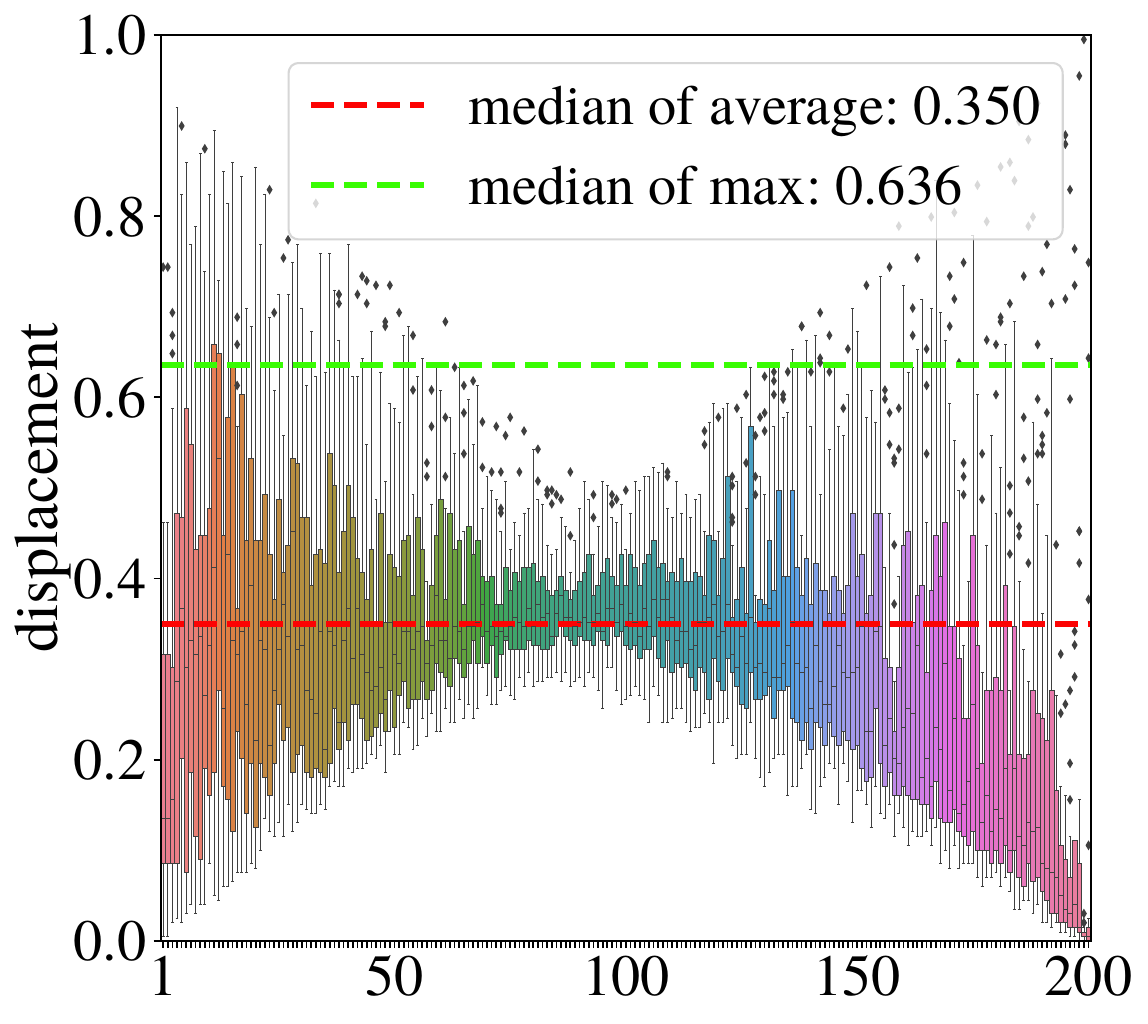}}
\hfill
\subfloat[$\SNR = 0.5$, unnormalized]{\includegraphics[width=0.45\columnwidth]{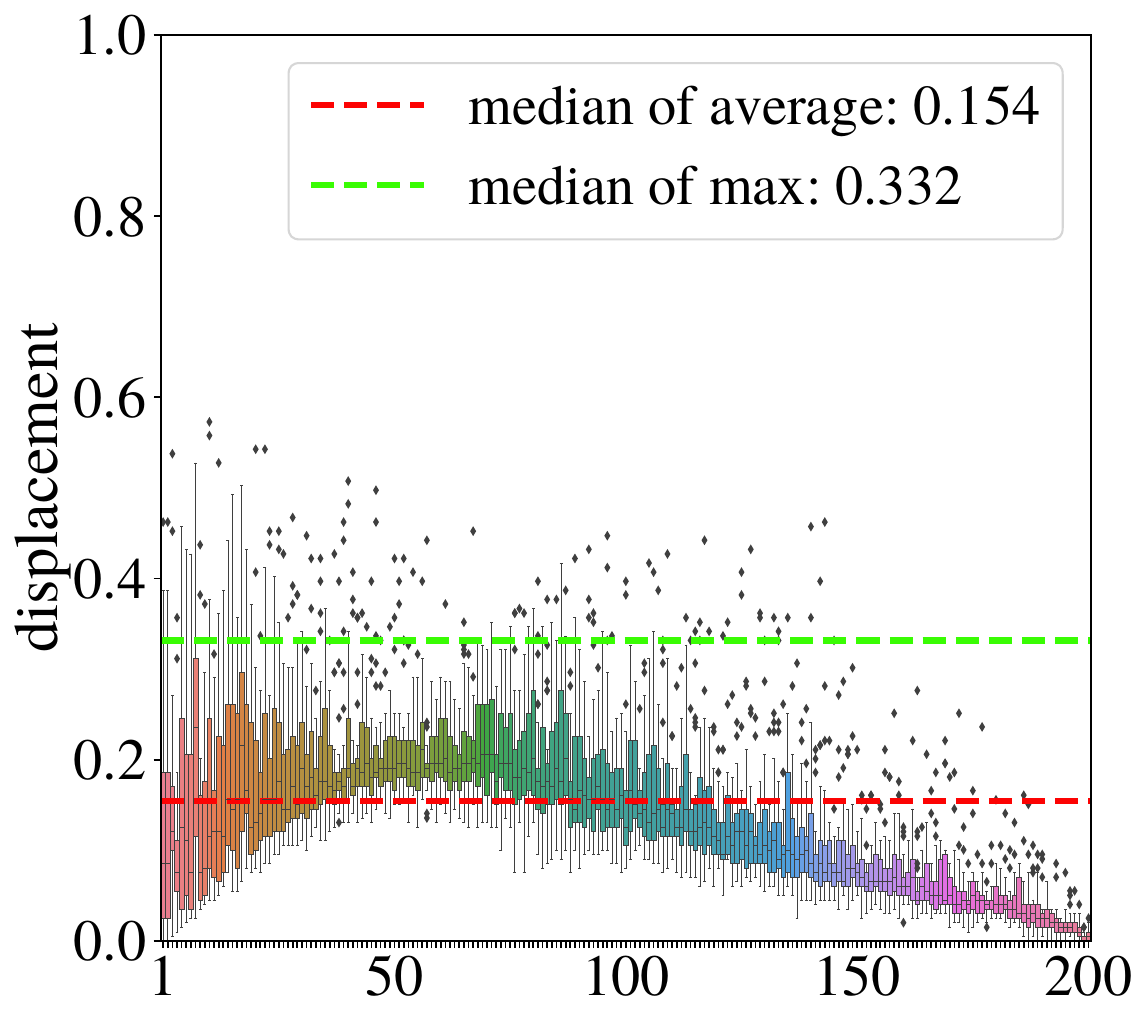}}
\hfill
\subfloat[$\SNR =0.26$, normalized]{\includegraphics[width=0.45\columnwidth]{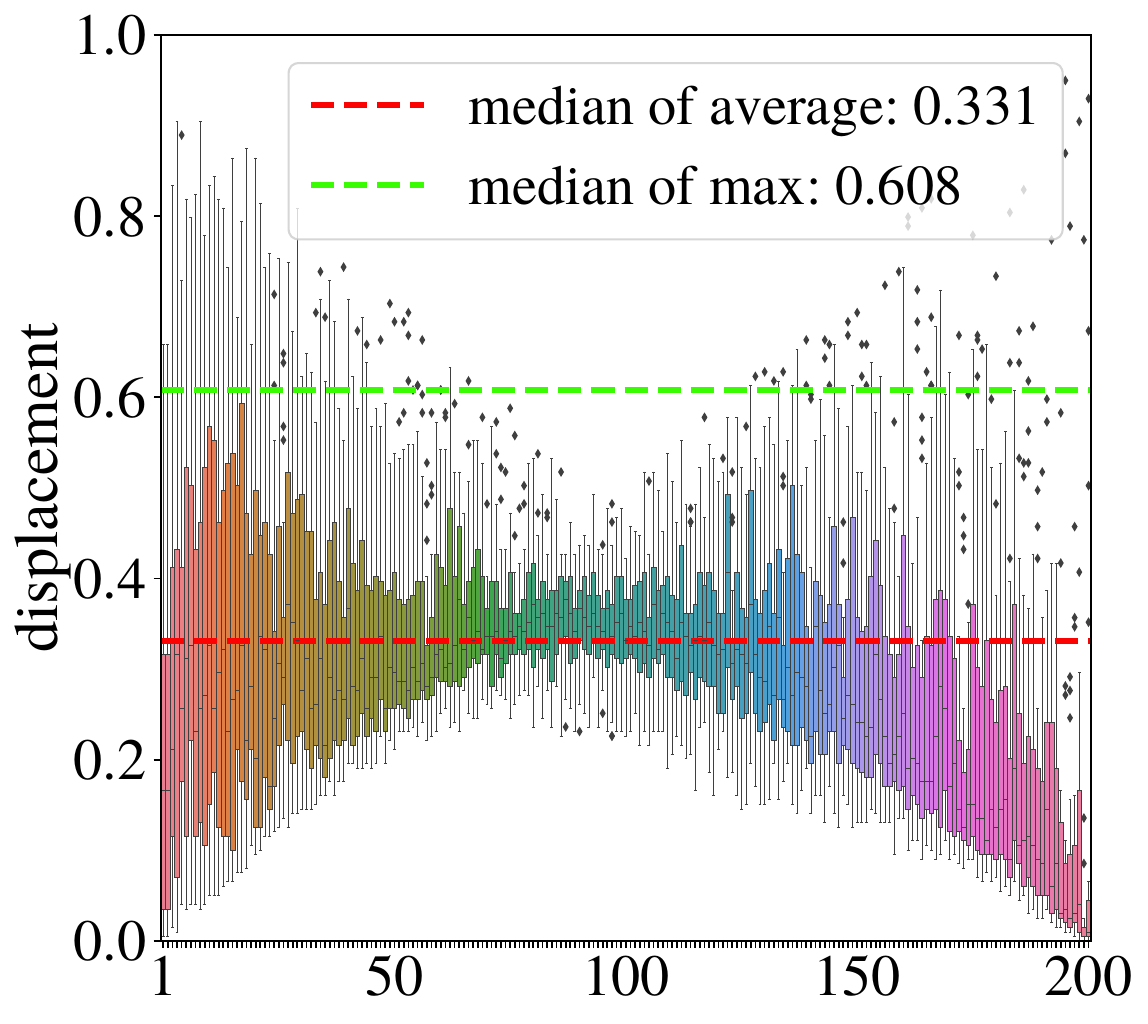}}
\hfill
\subfloat[$\SNR =0.5$, normalized]{\includegraphics[width=0.45\columnwidth]{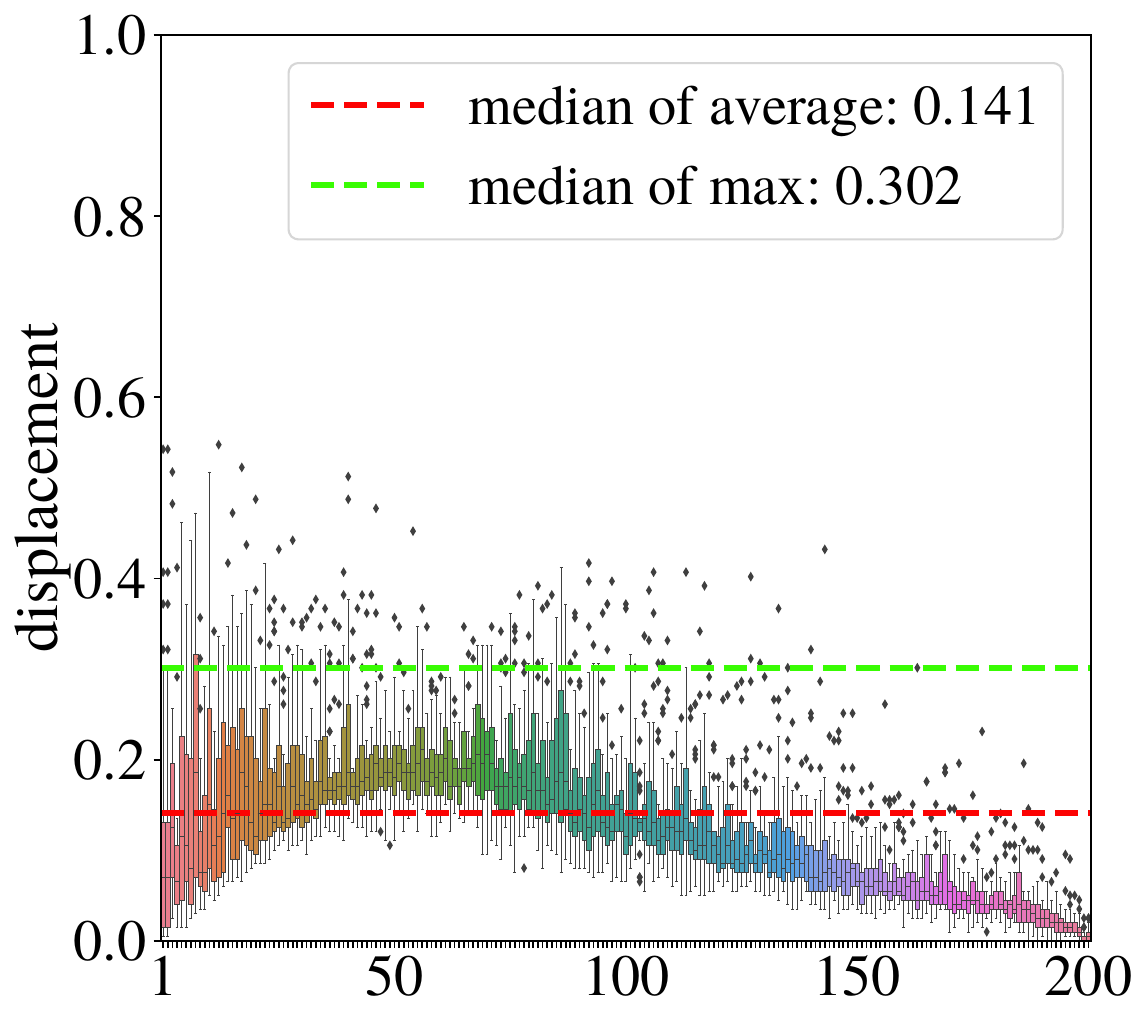}}
\hfill
\caption{Boxplot of displacement $\rho_{i}(\bpi,\widehat{\bpi})$, with ground-truth sorted Gamma distributed $\br$ ($n=200$) for both algorithms. The displacement $\rho_{i}(\bpi,\widehat{\bpi})$ decreases for both algorithms as the $\SNR$ grows. }
\label{fig:displacement}
\end{figure}

\appendix

\section{Proof of Theorem \ref{thm:mainHER}: unnormalized spectral ranking}\label{s:A}

\subsection{The expected measurement matrix}\label{ss:Hstar}
The noisy measurement $\BH$ can be decomposed into its signal and noise part
$\BH = \bar{\BH} + \BDelta$
as defined in~\eqref{def:H_bar} and~\eqref{def:Delta}.
Note that the singular value decomposition (SVD) of the expected measurement $\bar{\BH}$ is given by
\begin{align}
\label{eqn:svd of H_star}
    \bar{\BH} = \eta p (\br \bone_n - \bone_n \br^\top)= \bar{\sigma} \lp \bar{\bu}_1 \bar{\bu}_2^{\top} - \bar{\bu}_2 \bar{\bu}_1^{\top} \rp
\end{align}	
where
\[
\bar{\sigma} = \eta p\sqrt{n}\|\br - \alpha\bone_n\|, \qquad \bar{\bu}_1 =  -\frac{\bone_n}{\sqrt{n}}, \qquad \bar{\bu}_2 = \frac{\br-\alpha\bone_n}{\|\br-\alpha\bone_n\|}, \qquad\alpha = \frac{1}{n}\br^{\top}\bone_n.
\]

For any real anti-symmetric matrix $\BC$, $\mi \BC$ is Hermitian where $\mi$ is the imagery unit. Then there exist a unitary matrix $\BQ$ and a real diagonal matrix $\BSigma$ such that $\mi\BC = \BQ \BSigma \BQ^H$. 
The spectral decomposition of the Hermitian matrix $\mi\bar{\BH}$ is given by 
\[
\mi\bar{\BH} = \bar{\sigma} \lp \bar{\bphi}_1 \bar{\bphi}^{H}_1 - \bar{\bphi}_2 \bar{\bphi}^{H}_2 \rp
\] 
where 
\begin{align*}
\bar{\sigma} &= \eta p \sqrt{n}\norm{\br - \alpha \bone_n},~~ \bar{\bphi}_1 = \frac{1}{\sqrt{2}}\lp \frac{\br-\alpha\bone_n}{\|\br-\alpha\bone_n\|} - \mi\frac{\bone_n}{\sqrt{n}}\rp,~~\bar{\bphi}_2 = \frac{1}{\sqrt{2}}\lp  \frac{\br-\alpha\bone_n}{\|\br-\alpha\bone_n\|} + \mi\frac{\bone_n}{\sqrt{n}} \rp. \\
\end{align*}

\subsection{Proof of Theorem~\ref{thm:mainHER}}

Recall that the $\SNR$ in~\eqref{def:SNR} is approximately equal to $\bar{\sigma}/\|\BDelta\|$. 
Our proof relies on the following lemma.
\begin{lemma}\label{lem:delta}
Under Assumption~\ref{def:SNR}, then it holds with probability $1-o(1)$ that
\begin{equation}\label{eq:deltanorm}
\norm{\BDelta} \lesssim M\sqrt{ pn\log(n)},
\end{equation}
\begin{equation}\label{eq:delta_w}
\| \BDelta \bw\|_{\infty} \leq M \|\bw\|_{\infty}\sqrt{pn\log n}.
\end{equation}
for any fixed complex vector $\bw \in \mathbb{C}^{n}$ that is independent of $\BDelta$. 
\end{lemma}
Lemma~\ref{lem:delta} is proven in Section~\ref{ss:lemdelta}.
Note that under Assumption~\ref{def:SNR}, i.e., $\SNR \gtrsim 1, $
we have
\begin{equation}\label{eq:sigma_ratio}
\left|  \frac{\sigma}{\bar{\sigma}} - 1\right| \leq \frac{\|\BDelta\|}{\bar{\sigma}} \lesssim \frac{M\sqrt{pn\log n}}{\eta p\sqrt{n}\|\br - \alpha\bone_n\|}  \lesssim \SNR^{-1},~~~\SNR\lesssim \sqrt{\frac{\eta^2pn}{\log n}}
\end{equation}
following from Weyl's inequality (Theorem \ref{thm:weyl}) and $\|\BDelta\|\lesssim M\sqrt{pn\log n}$ in Lemma~\ref{lem:delta}.

\paragraph{Error decomposition.} As discussed in Section~\ref{ss:tech}, the $\ell_\infty$-norm error between $\bphi$ and $\bar{\bphi}$ up to some rotation $\beta\in\CC$ with $|\beta|=1$ can be decomposed as follows: 
\begin{equation}\label{eqn:unnormalized decomposition}
\begin{aligned}
& \norm{\bphi_1 - \beta\bar{\bphi}_1}_\infty \le \norm{\bphi_1 - \beta\widetilde{\bphi}_1}_\infty + \norm{\widetilde{\bphi}_1 - \bar{\bphi}_1}_\infty \\ 
& = \left\| \frac{\mi\BH \bphi_1}{\sigma}  - \frac{\mi \beta\BH\bar{\bphi}_1}{\sigma}\right\|_{\infty} + \left\|\frac{\mi\BH\bar{\bphi}_1}{\sigma} - \bar{\bphi}_1 \right\|_{\infty} \\
& = \frac{1}{\sigma} \|\BH(\bphi_1 - \beta\bar{\bphi})\|_{\infty} + \left\|\frac{\mi(\bar{\BH} + \BDelta)\bar{\bphi}_1}{\sigma} - \bar{\bphi}_1 \right\|_{\infty} \\
& = \underbrace{\frac{1}{\sigma} \|\bar{\BH}(\bphi_1 - \beta\bar{\bphi})\|_{\infty}}_{E_1} 
+ \underbrace{\frac{1}{\sigma} \|\BDelta(\bphi_1 - \beta\bar{\bphi})\|_{\infty}}_{E_2}  
+ \underbrace{\frac{|\sigma - \bar{\sigma}|}{\sigma} \|\bar{\bphi}_1\|_{\infty}}_{E_3} 
+ \underbrace{\frac{\|\BDelta\bar{\bphi}_1\|_{\infty}}{\sigma}}_{E_4}
 \end{aligned}
\end{equation}
where $\widetilde{\bphi}_1$ is defined in~\eqref{def:tphi1}. The estimations of each term are provided in Theorem~\ref{thm:E134} and~\ref{thm:E2}, which are justified in Section~\ref{ss:proofE134} and~\ref{ss:proofE2} respectively.
\begin{theorem}
\label{thm:E134}
Under the conditions in Theorem \ref{thm:mainHER}, with probability $1-o(1)$,
\begin{equation}
\begin{aligned}
E_1 & = \frac{1}{\sigma} \norm{\bar{\BH}(\bphi_1-\beta\bar{\bphi}_1)}_\infty \lesssim \SNR^{-1} \norm{\bar{\bphi}_1}_{\infty},\\
E_3 &= \frac{|\sigma - \bar{\sigma}|}{\sigma} \norm{\bar{\bphi}_1}_{\infty} \lesssim \SNR^{-1}  \norm{\bar{\bphi}_1}_{\infty},\\
E_4 &= \frac{1}{\sigma}\norm{\BDelta\bar{\bphi}_1}_\infty \lesssim \SNR^{-1} \norm{\bar{\bphi}_1}_{\infty}.
\end{aligned} 
\end{equation}
\end{theorem}

\begin{theorem}\label{thm:E2}
Under the conditions in Theorem \ref{thm:mainHER}, with probability $1-o(1)$,
\begin{equation}\label{eq:E2}
E_2 =  \frac{1}{\sigma} \norm{\BDelta(\bphi_1-\beta\bar{\bphi}_1)}_\infty  \lesssim \SNR^{-1}(\|\bphi_1\|_{\infty} + \|\bar{\bphi}_1\|_{\infty}).
\end{equation}
\end{theorem}

\begin{proof}[\bf Proof of Theorem~\ref{thm:mainHER} ]

Applying the bounds of $E_1,\cdots,E_4$, we have the $\ell_\infty$ error \eqref{eqn:unnormalized decomposition}  controlled by
\[
\begin{aligned}
\norm{\bphi_1 - \beta\bar{\bphi}_1}_\infty 
&\lesssim E_1 +E_2 +E_3 +E_4 \lesssim \SNR^{-1}(\|\bphi_1\|_{\infty} +\|\bar{\bphi}_1\|_{\infty} )
\end{aligned}
\]
with probability $1- o(1)$. Then under the assumption that $\SNR \gtrsim 1$, we have 
\[
\|\bar{\bphi}_1\|_{\infty} \leq \frac{1 + O(\SNR^{-1})}{1 - O(\SNR^{-1})} ~\|\bphi_1\|_{\infty}.
\]
Therefore, it holds that $\|\bphi_1 - \beta\bar{\bphi}_1\|_{\infty} \lesssim  \SNR^{-1} \|\bar{\bphi}_1\|_{\infty}  \lesssim  \SNR^{-1} \|\bphi_1\|_{\infty}$
for $\SNR^{-1}=O(1)$ which proves~\eqref{eq:HER_error} in Theorem~\ref{thm:mainHER}. 
Now we proceed to estimate~\eqref{eq:HER_error2}, i.e., $\min_{s\in \{\pm 1\}}\|\bx - s\bar{\bx}\|_{\infty}$ where $\bphi_1 = \bx + \mi\by$ and $\bar{\bphi}_1 = (\bar{\bx} + \mi\bar{\by})$. Here
\[
\bar{\bx}=\frac{\br-\alpha\bone_n}{\sqrt{2}\|\br - \alpha\bone_n\|},~~~\bar{\by} = -\frac{\bone_n}{\sqrt{2n}},~~~\alpha = \frac{\br^{\top}\bone_n}{n}.
\]

In particular, $\bx = \Re\bphi_1$ is obtained from Algorithm~\ref{algo:HER} and $\bx\perp \bar{\by}$ holds.
Denote $\beta = \exp(\mi\theta)=\lag \bar{\bphi}_1, \bphi_1\rag/|\lag \bar{\bphi}_1, \bphi_1\rag|$ and then 
\begin{align*}
\bphi_1 - \beta\bar{\bphi}_1 & = \bx + \mi\by - (\cos\theta + \mi\sin\theta)(\bar{\bx} + \mi\bar{\by})  = \left(\underbrace{[\bx,\by]}_{\BU} - \underbrace{[\bar{\bx},\bar{\by}]}_{\bar{\BU}} 
\underbrace{
\begin{bmatrix}
\cos\theta & \sin \theta \\
-\sin\theta &\cos\theta
\end{bmatrix}}_{\BQ} \right)
\begin{bmatrix}
1 \\
\mi
\end{bmatrix}
\end{align*}
Then
\[
\|\bx - (\bar{\bx}\cos\theta  - \bar{\by}\sin\theta )\|_{\infty} = \|\Real(\bphi_1 - \beta\bar{\bphi}_1)\|_{\infty}  \leq  \|\bphi_1 - \beta\bar{\bphi}_1\|_{\infty} \lesssim \SNR^{-1}\|\bar{\bphi}_1\|_{\infty}.
\]
From Davis-Kahan theorem, we have $\delta: = \| \bphi_1 - \beta\bar{\bphi}_1 \| = \|\BU - \bar{\BU}\BQ\|_F \lesssim \SNR^{-1}$ and by separating the real and imaginary part, we get
\[
\|\BU - \bar{\BU}\BQ\|_F \leq\delta  \Longrightarrow \|\BU - \bar{\BU}\BQ\| \leq  \delta \Longrightarrow \| 2\bar{\BU}^{\top}\BU-\BQ\| = 2 \| \bar{\BU}^{\top}(\BU - \bar{\BU}\BQ) \|\leq \sqrt{2}\delta
\]
where $\BU=[\bx,\by]$, $\bar{\BU} = [\bar{\bx},\bar{\by}]$, and $\bar{\BU}^{\top}\bar{\BU} = \I_2/2$. Note that
\[
\bar{\BU}^{\top}\BU = \begin{bmatrix}
\lag \bx, \bar{\bx}\rag  & \lag \by, \bar{\bx}\rag \\
0 &  \lag \by, \bar{\by}\rag
\end{bmatrix},~~\BQ = \begin{bmatrix}
\cos\theta & \sin \theta \\
-\sin\theta &\cos\theta
\end{bmatrix}
\]
where $\bx\perp\bar{\by}.$
This implies that $|\sin \theta | \leq \sqrt{2}\delta$ and $|\cos\theta|  \geq \sqrt{1-2\delta^2}.$
Then we have
\begin{align*}
\|\bx - \sign(\cos\theta) \bar{\bx} \|_{\infty} & \leq  \|\bx - (\bar{\bx} \cos\theta -  \bar{\by}\sin\theta)\|_{\infty} + (1 - |\cos\theta|)\|\bar{\bx}  \|_{\infty} + |\sin\theta| \cdot \|\bar{\by}\|_{\infty} \\
& \lesssim \SNR^{-1}\|\bar{\bphi}_1\|_{\infty} + (1-\sqrt{1-2\delta^2})\|\bar{\bx}  \|_{\infty} + \sqrt{2}\delta \|\bar{\by}\|_{\infty} \\
& \lesssim \SNR^{-1} \|\bar{\bx}\|_{\infty}
\end{align*}
where $\|\bar{\by}\|_{\infty}\lesssim \|\bar{\bx}\|_{\infty}$ and $\|\bar{\bphi}_1\|_{\infty} \leq \|\bar{\bx}\|_{\infty} + \|\bar{\by}\|_{\infty}\leq 2\|\bar{\bx}\|_{\infty}.$
\end{proof}

\subsection{Proof of Lemma \ref{lem:delta}}\label{ss:lemdelta}
\begin{proof}

{\bf Estimation of $\|\BDelta\|$ in~\eqref{eq:deltanorm}.} In the matrix form, $\BDelta$ is a sum of independent rank-2 random matrices:
\begin{align*}
\BDelta & = \sum_{i<j}\Delta_{ij}(\be_i\be_j^\top - \be_j\be_i^\top) = \sum_{i<j} \underbrace{\left[ (X_{ij}Y_{ij} - \eta p) (r_i - r_j) + X_{ij}(1-Y_{ij})Z_{ij}\right]}_{\Delta_{ij}}(\be_i\be_j^{\top} - \be_j\be_i^{\top})
\end{align*}
where $X_{ij} \sim$Bernoulli($p$), $Y_{ij}\sim$Bernoulli($\eta$), and $Z_{ij}\sim {\cal U}[-M,M].$

Note that $\E  \Delta_{ij} = 0$ and $|\Delta_{ij}|\lesssim M$.  The variance of each entry is bounded by
\begin{align*}
\E \Delta^2_{ij} & = (r_i - r_j)^2\E(X_{ij}Y_{ij} - \eta p )^2 + \E X_{ij}(1-Y_{ij})Z_{ij}^2 \\
& = (r_i - r_j)^2\eta p(1-\eta p ) + \frac{(1-\eta)pM^2}{3}  \leq 4\eta p  (1-\eta p )M^2 + \frac{(1-\eta)pM^2}{3} \lesssim pM^2
\end{align*}
where $\E Z_{ij}^2 = M^2/3$.
Then we have that
\[
\begin{aligned}
    \norm{\mathbb{E}\ls \sum_{i<j} \Delta^2_{ij}(\be_i\be_j^\top - \be_j\be_i^\top)(\be_i\be_j^\top - \be_j\be_i^\top)^\top  \rs} & = \norm{\mathbb{E}\ls \sum_{i<j} \Delta^2_{ij}(\be_i\be_i^\top+\be_j\be_j^\top)  \rs}\\
    &\lesssim  pM^2 \norm{n\I_n}=  npM^2.
\end{aligned}
\]
Note that $\norm{\Delta_{ij}(\be_i\be_j^\top-\be_j\be_i^\top)}\lesssim M$. Then the Bernstein inequality (Theorem \ref{thm:bernstein}) along with the assumption $\SNR \gtrsim 1$ gives that with probability $1- o(1)$,
\[
	\norm{\BDelta} \lesssim M\sqrt{ pn\log n} + M\log n \lesssim M\sqrt{ pn\log n}.
\]

\vskip0.25cm
\noindent{\bf Estimation of $\|\BDelta\bw\|_{\infty}$ in~\eqref{eq:delta_w}.} 
Each entry of $\BDelta\bw$ is a sum of $n$ independent entries:
$[\BDelta \bw]_k = \sum_{1\le j\le n} \Delta_{kj} w_j.$
It suffices to use Theorem \ref{thm:bernstein} to bound it: the variance of $[\BDelta \bw]_k$ is bounded by
\[
\E \ls \sum_{j=1}^n \Delta^2_{kj} |w_j|^2 \rs \lesssim pnM^2 \norm{\bw}_\infty^2,
\] 
where $\E |\Delta_{kj}|^2 \lesssim pM^2$.
Also we note that each term $|\Delta_{kj}w_j|$ is uniformly bounded by $2M\|\bw\|_{\infty}$. Theorem~\ref{thm:bernstein}, together with a union bound argument, implies that with probability $1-o(1)$, 
\[
\norm{\BDelta\bw}_{\infty} \lesssim M \| \bw \|_{\infty} \sqrt{pn\log n} + M\norm{\bw}_\infty \log n \lesssim M \|\bw\|_{\infty}\sqrt{pn\log n}.
\]
\end{proof}

\subsection{Proof of Theorem \ref{thm:E134}: the estimation of $E_1$, $E_3,$ and $E_4$}\label{ss:proofE134}

The Davis-Kahan bound yields
\[
\norm{\bphi_1 - \beta\bar{\bphi}_1} \lesssim \frac{\norm{\BDelta}}{\bar{\sigma} - \norm{\BDelta}} \lesssim  \SNR^{-1}  = O(1).
\] 

\paragraph{Estimation of $E_1$:} $E_1$ can be bounded by 
\begin{equation}
\begin{aligned}
E_1 &= \frac{1}{\sigma} \norm{\mi\bar{\BH}(\bphi_1-\beta\bar{\bphi}_1)}_\infty = \frac{\bar{\sigma}}{\sigma}\norm{\lp \bar{\bphi}_1 \bar{\bphi}_1^{H} - \bar{\bphi}_2 \bar{\bphi}_2^{H} \rp(\bphi_1-\beta\bar{\bphi}_1)}_\infty \\
&\lesssim \norm{\bphi_1-\beta\bar{\bphi}_1}(\norm{\bar{\bphi}_1}_{\infty}+\norm{\bar{\bphi}_2}_{\infty})  \lesssim  \SNR^{-1} \norm{\bar{\bphi}_1}_{\infty}
\end{aligned}
\end{equation}
where $\norm{\bar{\bphi}_1}_{\infty}=\norm{\bar{\bphi}_2}_{\infty}$.

\paragraph{Estimation of $E_3$:}  Note that
\[
E_3 = \frac{|\sigma - \bar{\sigma}|}{\sigma} \|\bar{\bphi}_1\|_{\infty} \lesssim \frac{\SNR^{-1}  \|\bar{\bphi}_1\|_{\infty}  }{1 - O(\SNR^{-1} )}
\]
which follows from~\eqref{eq:sigma_ratio}.

\paragraph{Estimation of $E_4$:}  The estimation of $E_4$ directly follows from~\eqref{eq:delta_w} that
\begin{equation}
E_4 = \frac{1}{\sigma} \norm{\BDelta\bar{\bphi}_1}_\infty \lesssim \frac{M\|\bar{\bphi}_1\|_{\infty}\sqrt{pn \log n}}{\eta p\sqrt{n}\|\br - \alpha\bone_n\|} \lesssim \SNR^{-1}  \|\bar{\bphi}_1\|_{\infty}
\end{equation}
with probability at least $1 - o(1)$ where $\sigma \approx \bar{\sigma}.$

\subsection{Proof of Theorem~\ref{thm:E2}: the estimation of $E_2$}\label{ss:proofE2}

Note that~\eqref{def:Tterm} implies that $E_2\leq T_1 + T_2$, and thus
we introduce the following lemma on the estimation of $T_1$ and $T_2$.
\begin{lemma}
\label{lem:Tterm}
Under Assumption~\ref{def:SNR}, then for $T_1$ and $T_2$ defined in \eqref{def:Tterm}, with probability $1-o(1)$,
\begin{align*}
T_{1} &= \sigma^{-1} \lvert \BDelta^\top_k(\bphi_1-\beta^{(k)}\bphi^{(k)}_1)\rvert \lesssim \SNR^{-2}\norm{\bphi_1^{(k)}}_{\infty} \lesssim \frac{\SNR^{-2}\|\bphi_1\|_{\infty}}{1 - O(\SNR^{-1})}, \\
T_{2} &=\sigma^{-1} \lvert \BDelta^\top_k(\beta^{(k)}\bphi^{(k)}_1  - \beta\bar{\bphi}_1)\rvert \lesssim \SNR^{-1}(\|\bphi_1\|_{\infty} + \|\bar{\bphi}_1\|_{\infty}).
\end{align*}
Then $E_2$ satisfies
\begin{align}
 E_2 \le T_1 + T_2 \lesssim \SNR^{-1}(\norm{\bphi_1}_{\infty} + \norm{\bar{\bphi}_1}_{\infty}).
\end{align}
\end{lemma}

\begin{proof}[\bf Proof of Lemma~\ref{lem:Tterm}] The proof uses Davis-Kahan theorem and also~\eqref{eq:delta_w} in Lemma~\ref{lem:delta}.
\noindent{\bf Estimation of $T_1$:} 
\begin{equation}
\begin{aligned}
T_1 &=  \sigma^{-1} \lvert \BDelta^\top_k(\bphi_1-\beta^{(k)}\bphi^{(k)}_1)\rvert \lesssim \frac{\norm{\BDelta}}{\sigma} \norm{\bphi_1-\beta^{(k)}\bphi^{(k)}_1} \lesssim  \SNR^{-1}\frac{\norm{\lp \BDelta - \BDelta^{(k)} \rp \bphi_1^{(k)}}}{\bar{\sigma} - (\|\BDelta\| + \|\BDelta^{(k)}\| ) }.
\end{aligned}
\end{equation}
Note that 
\[
\ls \lp \BDelta - \BDelta^{(k)} \rp \bphi_1^{(k)} \rs_{j} = \begin{cases}
    \Delta_{jk}\phi_{1,k}^{(k)}, \quad {j \neq k},\\
    -\BDelta_{k}^\top\bphi_1^{(k)}, \quad {j = k},
\end{cases}
\]
where $\BDelta_k$ is the $k$-th column of $\BDelta$ and $\BDelta$ is anti-symmetric.
As a result, we have
\[
\norm{\lp \BDelta - \BDelta^{(k)} \rp \bphi_1^{(k)}} \leq |\phi_{1,k}^{(k)}|\cdot\|\BDelta_k \| + \|\BDelta_k^{\top}\bphi_1^{(k)}\| \lesssim M\sqrt{pn\log n}\cdot \norm{\bphi_1^{(k)}}_{\infty}
\]
which follows from~\eqref{eq:delta_w} in Lemma~\ref{lem:delta}.
Thus, 
\begin{equation}\label{eq:phiphik}
\|\bphi_1 - \beta^{(k)}\bphi_1^{(k)}\| \lesssim \frac{M\sqrt{pn\log n}\cdot\norm{\bphi_1^{(k)}}_{\infty}}{\bar{\sigma} - 2\|\BDelta\|} \lesssim \frac{M\sqrt{pn\log n}\cdot\norm{\bphi_1^{(k)}}_{\infty}}{\bar{\sigma}} \lesssim \SNR^{-1} \|\bphi_1^{(k)}\|_{\infty}.
\end{equation}
As a result, it holds that
\begin{equation}\label{eq:phiphik2}
\|\bphi_1^{(k)}\|_{\infty} \leq \frac{\|\bphi_1\|_{\infty}}{1 - O(\SNR^{-1})}.
\end{equation}
Finally, we have $T_1$ bounded by 
\begin{equation}\label{eq:T1}
T_1 \lesssim \SNR^{-1}\frac{\norm{\lp \BDelta - \BDelta^{(k)} \rp \bphi_1^{(k)}}}{\bar{\sigma} - 2\|\BDelta\| } \lesssim \SNR^{-1}\cdot \frac{M\sqrt{pn\log n}\cdot \norm{\bphi_1^{(k)}}_{\infty}}{\bar{\sigma}}  \lesssim \SNR^{-2}\|\bphi_1\|_{\infty}.
\end{equation}

\noindent{\bf Estimation of $T_2$:} 
For $T_2$, we can directly apply~\eqref{eq:delta_w} in Lemma \ref{lem:delta} since $\bar{\bphi}_1$ and $\bphi^{(k)}_1 $ are independent of $\BDelta_k.$
\begin{equation*}
T_2 = \frac{1}{\sigma} \lvert \BDelta^\top_k(\beta^{(k)}\bphi^{(k)}_1  - \beta \bar{\bphi}_1)\rvert 
\leq \frac{1}{\sigma} ( \|\BDelta\bphi_1^{(k)}\|_{\infty} +  \|\BDelta\bar{\bphi}_1 \|_{\infty}) \lesssim \SNR^{-1} ( \|\bphi_1^{(k)}\|_{\infty} +  \|\bar{\bphi}_1 \|_{\infty} ).
\end{equation*}
By~\eqref{eq:phiphik2}, we have
\begin{equation}\label{eq:T2}
T_2 \lesssim \SNR^{-1} ( \|\bphi_1\|_{\infty} +  \|\bar{\bphi}_1 \|_{\infty} )
\end{equation}
provided that $\SNR\gtrsim 1.$
\end{proof}

\section{Proof of Theorem \ref{thm:mainNHER}: normalized spectral ranking}\label{s:B}

\subsection{The expected measurement matrix}\label{ss:Hstar_norm}
We define the degree matrix of the noisy measurement matrix $\BD$ and its expectation $\bar{\BD}$ as
\begin{equation}\label{def:ED}
D_{ii} = \sum_{j=1}^n |H_{ij}|, \qquad \bar{D}_{ii} = \E\sum_{j=1}^n |H_{ij}|, \quad 1 \le i \le n.
\end{equation}

Algorithm~\ref{algo:NHER} is based on the left normalized measurement:
$\BH_{\rm L} = \BD^{-1} \BH.$ 
To understand the eigenvectors of $\BH_{\rm L}$, we first look into its population counterpart:
\[
\bar{\BH}_{\rm L} = (\E \BD)^{-1}\E \BH= \bar{\BD}^{-1} \bar{\BH}.
\]
Note that the eigenvalues/eigenvectors of $\BH_{\rm L}$ and $\bar{\BH}_{\rm L}$ are closely related to its symmetrized version:
\[
\BH_{\rm sym} : = \BD^{-1/2} \BH \BD^{-1/2},~~~\bar{\BH}_{\rm sym} := \bar{\BD}^{-1/2} \bar{\BH} \bar{\BD}^{-1/2}.
\]
The symmetric normalized measurement matrix is a perturbed version of $\bar{\BH}_{\rm sym}$ with the perturbation being
\begin{equation}\label{def:Delta_sym}
 \BDelta_{\rm sym} = \BD^{-1/2} \BH \BD^{-1/2}  - \bar{\BD}^{-1/2} \bar{\BH} \bar{\BD}^{-1/2}.
\end{equation}
The SVD of the rank-2 expected measurement matrix $\bar{\BH}_{\rm sym}$ is given by
\begin{equation}\label{eq:barH_sym_svd}
\bar{\BH}_{\rm sym} = \bar{\BD}^{-1/2}(\br \bone_n^{\top} - \bone_n\br^{\top})\bar{\BD}^{-1/2}= \bar{\xi} \lp \bar{\bv}_1\bar{\bv}_2^{\top} - \bar{\bv}_2 \bar{\bv}_1^{\top} \rp
\end{equation}
where 
\[
\begin{aligned}
& \bar{\xi} = \eta p \|\bar{\BD}^{-1/2}(\br - \gamma\bone_n)\|\norm{\bar{\BD}^{-1/2}\bone_n} \\ 
& \bar{\bv}_1 = -\frac{\bar{\BD}^{-1/2}\bone_n}{\norm{\bar{\BD}^{-1/2}\bone_n}}, \quad \bar{\bv}_2 = \frac{\bar{\BD}^{-1/2}(\br - \gamma\bone_n)}{\norm{\bar{\BD}^{-1/2}(\br - \gamma\bone_n)}}, \quad \gamma = \frac{\br^\top \bar{\BD}^{-1}\bone_n}{\bone_n^\top \bar{\BD}^{-1}\bone_n}.
\end{aligned}
\]
The parameter $\gamma$ is chosen so that $\lag \bar{\bv}_1, \bar{\bv}_2\rag = 0.$

Since $\bar{\BH}_{\rm sym}$ is anti-symmetric, $\mi\bar{\BH}_{\rm sym}$ is a Hermitian matrix whose spectral decomposition is given by
\[
\mi\bar{\BH}_{\rm sym} = \bar{\xi} \lp \bar{\bvphi}_1 \bar{\bvphi}_1^{H} - \bar{\bvphi}_2 \bar{\bvphi}_2^{H} \rp \quad {\rm where} \quad \bar{\bvphi}_1 = \frac{1}{\sqrt{2}}(\bar{\bv}_2 + \mi\bar{\bv}_1), \quad \bar{\bvphi}_2 = \frac{1}{\sqrt{2}}(\bar{\bv}_2 - \mi\bar{\bv}_1).
\]
As a result, $\mi\bar{\BH}_{\rm L}$ equals 
\[
\mi\bar{\BH}_{\rm L} = \mi \bar{\BD}^{-1/2}\bar{\BH}_{\rm sym}\bar{\BD}^{1/2}= \bar{\xi} \lp \bar{\BD}^{-1/2}\bar{\bvphi}_1 \bar{\bvphi}_1^{H}\bar{\BD}^{1/2} - \bar{\BD}^{-1/2}\bar{\bvphi}_2 \bar{\bvphi}_2^{H}\bar{\BD}^{1/2} \rp
\]
and its normalized eigenvectors are
\begin{align*}
\bar{\bpsi}_1 = \frac{\bar{\BD}^{-1/2}\bar{\bvphi}_1}{\|\bar{\BD}^{-1/2}\bar{\bvphi}_1\|}, \text{ with eigenvalue }\bar{\xi}, \qquad \bar{\bpsi}_2 = \frac{\bar{\BD}^{-1/2}\bar{\bvphi}_2}{\|\bar{\BD}^{-1/2}\bar{\bvphi}_2\|}, \text{ with eigenvalue }-\bar{\xi}.
\end{align*}

\subsection{Proof of Theorem~\ref{thm:mainNHER}}
Similarly, we need a proper measure of the signal-to-noise ratio (SNR). We define the SNR for the normalized spectral ranking as
\begin{equation}\label{def:SNRN}
\SNR_N(\eta,p,n,\br,M) = \frac{\bar{\xi}}{\|\BDelta_{\rm  sym}\|}.
\end{equation}
For simplicity, we abbreviate $\SNR_N(\eta,p,n,\br,M)$ to $\SNR_N.$
To better understand this $\SNR_N$, we introduce the following lemma on the degree $\BD$ and the normalized noise $\|\BDelta_{\rm sym}\|$.
\begin{lemma}\label{lem:delta_sym}
Under Assumption~\ref{def:SNR2}, then with probability $1-o(1)$, we have
\begin{align*}
& \lambda pnM\I_n \preceq \bar{\BD}\preceq 2pnM, \\
& \norm{\BD - \bar{\BD}} \lesssim M\sqrt{pn\log(n)},~~\| \bar{\BD}^{-1}\BD - \I_n \| \leq 1/2 \\
& \kappa(\bar{\BD}) \asymp \kappa(\BD) \lesssim \frac{1}{\lambda}, \\
& \norm{\BDelta_{\rm sym}}   \lesssim \frac{1}{\lambda^2}\sqrt{\frac{\log n}{pn}}.
\end{align*}
\end{lemma}

Note that
\begin{equation}\label{eq:SNR_norm}
 \bar{\xi}  = \eta p \|\bar{\BD}^{-1/2}(\br - \gamma\bone_n)\|\norm{\bar{\BD}^{-1/2}\bone_n} \gtrsim \frac{\eta p}{pn M} \| \br - \gamma\bone_n\|\norm{\bone_n} \geq \frac{\bar{\sigma}}{pn M}
\end{equation}
where $\|\br - \gamma\bone_n\|\geq \|\br - \alpha\bone_n\|.$ Also we note that
\begin{equation}\label{eq:barxi_low}
\bar{\xi} \lesssim \eta p \cdot \frac{\|\br - \gamma \bone_n\| \|\bone_n\|}{\lambda p nM} \lesssim \eta p \cdot \frac{nM}{\lambda pnM} \lesssim \frac{\eta}{\lambda}
\end{equation}
where $\|\br - \gamma \bone_n\|\lesssim \sqrt{n}M.$
Then $\SNR_N$ is lower bounded by
\begin{equation}\label{def:SNR_norm}
\SNR_N(\eta,p,n,\br,M) \gtrsim  \lambda^2\bar{\xi} \sqrt{\frac{pn}{\log n}}\gtrsim\frac{\lambda^2\bar{\sigma}}{pnM} \cdot \sqrt{\frac{pn}{\log n}} = \lambda^2\sqrt{\frac{\eta^2 pn}{\log n}}\cdot \frac{\|\br - \alpha\bone_n\|}{M\sqrt{n}} =\lambda^2 \SNR
\end{equation}
where $\bar{\sigma} = \eta p \sqrt{n}\|\br - \alpha\bone_n\|.$

\vskip0.25cm

\paragraph{Error decomposition.} Now we proceed to analyze the $\ell_{\infty}$-norm perturbation bound between $\bpsi_1$ and $\bar{\bpsi}_1$. 
Similar to the unnormalized case, $\|\bpsi_1 - \beta\bar{\bpsi}_1\|_{\infty}$ can be decomposed into
\begin{equation}\label{eq:error_norm}
\begin{aligned}
\norm{\bpsi_1 - \beta\bar{\bpsi}_1}_\infty & \leq \|\bpsi_1 - \beta\widetilde{\bpsi}_1\|_{\infty} + \| \widetilde{\bpsi}_1 - \bar{\bpsi}_1 \|_{\infty} \le \norm{\bpsi_1 - \frac{\mi\beta\BD^{-1}\BH\bar{\bpsi}_1}{\xi}}_\infty + \norm{ \frac{\mi\BD^{-1}\BH\bar{\bpsi}_1}{\xi} - \bar{\bpsi}_1}_\infty \\
& \le \underbrace{\frac{1}{\xi} \norm{\BD^{-1}\bar{\BH}(\bpsi_1-\beta\bar{\bpsi}_1)}_\infty}_{E_1} + \underbrace{\frac{1}{\xi} \norm{ \BD^{-1}\BDelta(\bpsi_1-\beta\bar{\bpsi}_1)}_\infty}_{E_2} \\
&\qquad + \underbrace{\frac{|\bar{\xi}-\xi|}{\xi}\norm{\bar{\bpsi}_1}_{\infty}}_{E_3} + \underbrace{ \frac{1}{\xi} \norm{\lp \BD^{-1}\BH -\bar{\BD}^{-1}\bar{\BH}  \rp\bar{\bpsi}_1}_\infty}_{E_4},
\end{aligned}
\end{equation}
as shown in~\eqref{eq:psi_dec1}.
Then we will prove the following theorems, which will be used to prove Theorem~\ref{thm:mainNHER}. The proofs of Theorem~\ref{thm:E134_norm} and~\ref{thm:E2_norm} are deferred to Section~\ref{ss:proofE134N} and~\ref{ss:proofE2N} respectively.
\begin{theorem}
\label{thm:E134_norm}
Under Assumption~\ref{def:SNR2}, it holds with probability $1-o(1)$,
\begin{equation}
\begin{aligned}
E_1 & = \frac{1}{\xi} \norm{ \BD^{-1}\bar{\BH}(\bpsi_1-\beta\bar{\bpsi}_1)}_\infty \lesssim \lambda^{-3}\SNR^{-1} \|\bar{\bpsi}_1\|_{\infty},\\
E_3 &= \frac{|\bar{\xi}-\xi|}{\xi}\norm{\bar{\bpsi}_1}_{\infty} \lesssim \lambda^{-2}\SNR^{-1} \|\bar{\bpsi}_1\|_{\infty}, \\
E_4 &= \frac{1}{\xi} \norm{\lp \BD^{-1}\BH -\bar{\BD}^{-1}\bar{\BH}  \rp\bar{\bpsi}_1}_\infty\lesssim  \lambda^{-1}\SNR^{-1}\|\bar{\bpsi}_1\|_{\infty}.
\end{aligned}
 \end{equation}
\end{theorem}
\begin{theorem}\label{thm:E2_norm}
Under Assumption~\ref{def:SNR2}, it holds with probability $1-o(1)$,
\begin{equation}
E_2 = \frac{1}{\xi} \norm{\BD^{-1}\BDelta(\bpsi_1-\beta\bar{\bpsi}_1)}_\infty  \lesssim \lambda^{-3}\SNR^{-1}(\|\bpsi_1\|_{\infty} + \|\bar{\bpsi}_1\|_{\infty}).
\end{equation}
\end{theorem}

\begin{proof}[\bf Proof of Theorem~\ref{thm:mainNHER}]
Using~\eqref{eq:error_norm}, we have
\[
\|\bpsi_1 - \beta \bar{\bpsi}_1\|_{\infty} \leq \sum_{\ell=1}^4 E_{\ell} \lesssim \lambda^{-3}\SNR^{-1}( \|\bpsi_1\|_{\infty} + \|\bar{\bpsi}_1\|_{\infty} )
\] 
which actually implies
\[
\|\bpsi_1 - \beta \bar{\bpsi}_1\|_{\infty} \lesssim \lambda^{-3}\SNR^{-1}\|\bar{\bpsi}_1\|_{\infty}.
\]
For the error bound on $\min_{s\in\{\pm 1\}}\|\bx - s\bar{\bx}\|$ where $\bpsi_1 = \bx + \mi\by$ and $\bar{\bpsi}_1 = \bar{\bx} + \mi\bar{\by}$:
\[
\bar{\bpsi}_1 \propto \bar{\BD}^{-1/2} (\bar{\bv}_2 + \mi \bar{\bv}_1) \propto  \frac{\bar{\BD}^{-1}(\br - \gamma\bone_n)}{\norm{\bar{\BD}^{-1/2}(\br - \gamma\bone_n)}} - \mi\frac{\bar{\BD}^{-1}\bone_n}{\norm{\bar{\BD}^{-1/2}\bone_n}},
\]
i.e., $\bar{\bx}\propto \bar{\BD}^{-1}(\br - \gamma\bone_n)$ and $\bar{\by}\propto \bar{\BD}^{-1}\bone_n$. This also implies $\bar{\bx}^{\top}\bar{\BD}\bar{\bx} = \bar{\by}^{\top}\bar{\BD}\bar{\by}$ and
\[
\|\bar{\by}\| \leq \frac{\|\bar{\bx}\|}{\sqrt{\lambda}},~~~ \|\bar{\by}\|_{\infty} \leq \frac{\|\bar{\bx}\|_{\infty}}{\lambda}.
\]

Note that Algorithm~\ref{algo:NHER} outputs $\bx=\Re\bpsi_1 $ with $\bx\perp \bone_n$.
Denote $\beta = \exp(\mi\theta)=\lag \bar{\bpsi}_1, \bpsi_1\rag/|\lag \bar{\bpsi}_1, \bpsi_1\rag|$ and then 
\begin{align*}
\bpsi_1 - \beta\bar{\bpsi}_1 & = \left( \underbrace{[\bx,\by]}_{\BU} - \underbrace{[\bar{\bx},\bar{\by}] }_{\bar{\BU}}
\underbrace{\begin{bmatrix}
\cos\theta & \sin \theta \\
-\sin\theta &\cos\theta
\end{bmatrix}}_{\BQ}\right) 
\begin{bmatrix}
1 \\
\mi
\end{bmatrix}
\end{align*}
Here $\BU$ and $\bar{\BU}$ satisfy
\[
\BU^{\top}\BD\BU  = 
\begin{bmatrix}
\lag \bx,\BD\bx\rag & 0 \\
0 & \lag \by,\BD\by\rag
\end{bmatrix},~
\bar{\BU}^{\top}\bar{\BD}\bar{\BU}
 = 
\begin{bmatrix}
\lag \bar{\bx},\bar{\BD}\bar{\bx}\rag & 0 \\
0 & \lag \bar{\by},\bar{\BD}\bar{\by}\rag
\end{bmatrix}
,~~\|\bx\|^2 + \|\by\|^2=1.
\]
Note that
\[
\|\bx - (\cos\theta \bar{\bx} - \sin\theta \bar{\by})\|_{\infty} = \|\Real(\bpsi_1 - \beta\bar{\bpsi}_1)\|_{\infty}  \leq \|\bpsi_1 - \beta\bar{\bpsi}_1\|_{\infty} \leq \lambda^{-3}\SNR^{-1}\|\bar{\bpsi}_1\|_{\infty}
\]
From the Davis-Kahan theorem, we have $\delta: = \| \bpsi_1 - \beta\bar{\bpsi}_1 \| \lesssim \lambda^{-5/2}\SNR^{-1}$, following from~\eqref{eq:phi1_dk}. By separating the real and imaginary part, we get $\|\BU - \bar{\BU}\BQ\|\leq \|\BU - \bar{\BU}\BQ\|_F \leq\delta$, and then it holds that
\begin{align*}
& \left\|  (\bar{\BU}^{\top} \bar{\BD}\bar{\BU})^{-1}
  \bar{\BU}^{\top}\bar{\BD}\BU-\BQ\right\| 
  = \left\|  (\bar{\BU}^{\top} \bar{\BD}\bar{\BU})^{-1}
  \bar{\BU}^{\top}\bar{\BD} (\BU-\bar{\BU}\BQ )\right\|  \\
&\qquad \leq \delta \left\|  (\bar{\BU}^{\top}\bar{\BD}\bar{\BU})^{-1}
  \bar{\BU}^{\top}\bar{\BD} \right\|  \leq \delta \sqrt{ \frac{\|\bar{\BD}\bar{\bx}\|^2}{|\lag \bar{\bx},\bar{\BD}\bar{\bx}\rag|^2} + \frac{\|\bar{\BD}\bar{\by}\|^2}{|\lag \bar{\by},\bar{\BD}\bar{\by}\rag|^2} } \lesssim \frac{\delta}{\lambda}
\end{align*}
since the condition number of $\bar{\BD}$ is of order $O(1/\lambda).$

Note that
\[
 (\bar{\BU}^{\top}\bar{\BD}\bar{\BU})^{-1}  \bar{\BU}^{\top}\bar{\BD}\BU 
= \begin{bmatrix}
 \frac{ \lag \bx, \bar{\BD}\bar{\bx}\rag}{\lag \bar{\bx},\bar{\BD}\bar{\bx}\rag} &\frac{\lag \by,\bar{\BD} \bar{\bx}\rag}{\lag \bar{\bx},\bar{\BD}\bar{\bx}\rag} \\
0 & \frac{\lag \by,\bar{\BD}\bar{\by}}{\lag \bar{\by},\bar{\BD}\bar{\by}\rag  }\rag
\end{bmatrix},~~\BQ = \begin{bmatrix}
\cos\theta & \sin \theta \\
-\sin\theta &\cos\theta
\end{bmatrix}
\]
where $\bx\perp\bar{\BD}\bar{\by}$ and $\bar{\BD}\bar{\by} \propto \bone_n.$
This implies that $|\sin \theta | \leq \lambda^{-1}\delta$ and $|\cos\theta|  \geq \sqrt{1-\lambda^{-2}\delta^2}.$
Then following the same argument in the proof of Theorem~\ref{thm:mainHER}, we have
\begin{align*}
\|\bx - \sign(\cos\theta) \bar{\bx} \|_{\infty} &  \lesssim \lambda^{-3} \SNR^{-1}\|\bar{\bpsi}_1\|_{\infty} + (1-\sqrt{1-\lambda^{-2}\delta^2})\|\bar{\bx}  \|_{\infty} + \lambda^{-1}\delta \|\bar{\by}\|_{\infty} \\ & \lesssim \lambda^{-5}\SNR^{-1} \|\bar{\bx}\|_{\infty}
\end{align*}
where $\|\bar{\by}\|_{\infty}\lesssim \lambda^{-1} \|\bar{\bx}\|_{\infty}$ and $\|\bar{\bpsi}_1\|_{\infty} \leq \|\bar{\bx}\|_{\infty} + \|\bar{\by}\|_{\infty}$. The bound also holds for $\ell_2$-norm, following from a similar argument. As a result, it holds that
\begin{align*}
\min_{s\in\{\pm 1\}}\left\| \frac{\BD\bx}{\|\BD\bx\|} - \frac{s \bar{\BD}\bar{\bx}}{\|\bar{\BD}\bar{\bx}\|} \right\|_{\infty} 
& \leq \frac{ \left\| \BD(\bx- s \bar{\bx}) \right\|_{\infty} }{\|\BD\bx\|}
 +  \frac{| \|\BD\bx\| -  \|\bar{\BD}\bar{\bx}\| | \cdot \|\BD\bar{\bx}\|_{\infty} }{\|\BD\bx\|\|\bar{\BD}\bar{\bx}\|} 
 + \frac{\left\|  (\BD -  \bar{\BD})\bar{\bx} \right\|_{\infty} }{\|\bar{\BD}\bar{\bx}\|} \\
 & \lesssim \frac{\SNR^{-1}}{\lambda^7}\frac{\|\bar{\bx}\|_{\infty}}{\|\bar{\bx}\|} \lesssim \frac{\SNR^{-1}}{\lambda^8}\frac{\|\bar{\BD}\bar{\bx}\|_{\infty}}{\|\bar{\BD}\bar{\bx}\|} 
\end{align*}
where
\begin{align*}
 \frac{ \left\| \BD(\bx- s \bar{\bx}) \right\|_{\infty} }{\|\BD\bx\|} & \lesssim \frac{1}{\lambda} \frac{ \left\| \bx- s \bar{\bx} \right\|_{\infty} }{\|\bx\|} \lesssim \frac{\SNR^{-1}}{\lambda^6} \frac{\|\bar{\bx}\|_{\infty}}{\|\bar{\bx}\|}, \\
 \frac{| \|\BD\bx\| -  \|\bar{\BD}\bar{\bx}\| | \cdot \|\BD\bar{\bx}\|_{\infty} }{\|\BD\bx\|\|\bar{\BD}\bar{\bx}\|} 
& \leq \frac{\|\BD(\bx - s\bar{\bx})\| + \|(\BD - \bar{\BD})\bar{\bx}\|}{d_{\min}\bar{d}_{\min}\|\bx\|\|\bar{\bx}\|} \cdot d_{\max}\|\bar{\bx}\|_{\infty} \\
& \lesssim \frac{1}{\lambda^2} \frac{\|\bx - s\bar{\bx}\|}{\|\bar{\bx}\|\|\bx\|}\cdot \|\bar{\bx}\|_{\infty} + \frac{1}{\lambda} \sqrt{\frac{\log n}{pn}} \frac{\|\bar{\bx}\|_{\infty}}{\|\bx\|} \lesssim \frac{\SNR^{-1}}{\lambda^7} \frac{\|\bar{\bx}\|_{\infty}}{\|\bar{\bx}\|}, \\
 \frac{\left\|  (\BD -  \bar{\BD})\bar{\bx} \right\|_{\infty} }{\|\bar{\BD}\bar{\bx}\|} 
 & \leq \frac{\|\BD-\bar{\BD}\|\|\bar{\bx}\|_{\infty}}{\bar{d}_{\min}\|\bar{\bx}\|} \lesssim \frac{1}{\lambda} \sqrt{\frac{\log n}{pn}} \frac{\|\bar{\bx}\|_{\infty}}{\|\bar{\bx}\|} 
 \lesssim \frac{\SNR^{-1}}{\lambda}\frac{\|\bar{\bx}\|_{\infty}}{\|\bar{\bx}\|}.
\end{align*}
The inequalities above follow from the following facts:
\[
\frac{\|\BD - \bar{\BD}\|}{d_{\min}}\lesssim \frac{1}{\lambda}\sqrt{\frac{\log n}{pn}} \lesssim \frac{\SNR^{-1}}{\lambda},~~~d_{\min} \asymp \bar{d}_{\min},~~~\frac{d_{\max}}{d_{\min}} = O(\lambda^{-1}).
\]

\end{proof}

\subsection{Proof of Lemma~\ref{lem:delta_sym}}\label{ss:lemdeltaN}
\noindent(a) Note that $\bar{\BD}$ satisfies
\begin{equation}\label{def:barD}
\bar{D}_{ii} = \E D_{ii} = \sum_{j=1}^n \E |H_{ij}| = p\eta \sum_{j=1}^n |r_i - r_j| + \frac{p(1-\eta)(n-1)M}{2}
\end{equation}
where the expectation of $|H_{ij}|$ is
\begin{align*}
\E |H_{ij}| & = \E \left| X_{ij}Y_{ij}(r_i - r_j) + X_{ij}(1- Y_{ij})Z_{ij} \right| = p \eta \left| r_i - r_j \right| + p(1-\eta)\E \left| Z_{ij} \right|    \\
& = 
\begin{cases}
p \eta |r_i - r_j| + p(1-\eta)M/2, & j\neq i, \\
0, & j = i.
\end{cases}
\end{align*}
where $Z_{ij}\sim{\cal U}[-M,M].$ 
By the definition of~\eqref{def:lambda}, it holds that
\[
\lambda p nM \leq \bar{\BD}_{ii} \leq 2pnM,~~\forall 1\leq i\leq n.
\]

\noindent(b) and (c) Note that $D_{ii} - \bar{D}_{ii} = \sum_{j=1}^n (|H_{ij}| - \E |H_{ij}|)$
where
$\Var(|H_{ij}|) \leq \E |H_{ij}|^2 = p\eta (r_i - r_j)^2 + p(1-\eta)M^2/3\lesssim pM^2$ and $|H_{kj}|\lesssim M.$
Then the Bernstein inequality gives
\[
\max_{1\leq i\leq n}| D_{ii} - \bar{D}_{ii}| \lesssim M \sqrt{pn\log n} + M\log n \lesssim M \sqrt{pn\log n}  \leq \frac{1}{2}\lambda pnM 
\]
under $pn /\log n\gtrsim \lambda^{-2}$ in Assumption~\ref{def:SNR2}. Thus
\[
\|\BD - \bar{\BD}\| \lesssim \frac{1}{2} \min_{1\leq i\leq n}\bar{D}_{ii} \Longrightarrow \frac{1}{2}\min_{1\leq i\leq n}\bar{D}_{ii} \leq D_{kk} \leq \frac{3}{2}\max_{1\leq i\leq n}\bar{D}_{ii},~~\forall 1\leq k\leq n
\]
which implies 
\[
\kappa(\BD) \asymp \kappa(\bar{\BD}) = O(1/\lambda),~~~\| \bar{\BD}^{-1}\BD - \I_n \| \leq 1/2.
\]

\noindent(d) 
The symmetric normalized error $\norm{\BDelta_{\rm sym}}$ can be decomposed into
\[
\begin{aligned}
\norm{\BDelta_{\rm sym}} & = \| \BD^{-1/2}\BH\BD^{-1/2} - \bar{\BD}^{-1/2}\bar{\BH}\bar{\BD}^{-1/2}  \| \\
& \leq \norm{\BD^{-1/2}\BDelta\BD^{-1/2}} + \norm{\BD^{-1/2}\bar{\BH}\BD^{-1/2} - \bar{\BD}^{-1/2}\bar{\BH}\bar{\BD}^{-1/2}}.
\end{aligned}
\]
The first term is upper bounded by 
\[
\norm{\BD^{-1/2}\BDelta\BD^{-1/2}} \le \frac{\norm{\BDelta}}{d_{\min}} \lesssim \frac{M\sqrt{pn \log n}}{\bar{d}_{\min}} = \frac{1}{\lambda}\sqrt{\frac{\log n}{pn}}
\]
where $\|\BDelta\|\lesssim M\sqrt{pn\log n}$ and $d_{\min}\gtrsim \lambda pnM.$
The second term is 
\[
\begin{aligned}
&\norm{\BD^{-1/2}\bar{\BH}\BD^{-1/2} - \bar{\BD}^{-1/2}\bar{\BH}\bar{\BD}^{-1/2}} \\ 
&\le \norm{\BD^{-1/2}\bar{\BH}(\BD^{-1/2} - \bar{\BD}^{-1/2})} + \norm{(\BD^{-1/2} - \bar{\BD}^{-1/2})\bar{\BH}\bar{\BD}^{-1/2}} \\
&\le \norm{\BD^{-1/2}\bar{\BH}(\BD - \bar{\BD})(\BD\bar{\BD}^{1/2}+\bar{\BD}\BD^{1/2})^{-1}} + \norm{(\BD\bar{\BD}^{1/2}+\bar{\BD}\BD^{1/2})^{-1}(\BD - \bar{\BD})\bar{\BH}\bar{\BD}^{-1/2}} \\
& \leq \left(\frac{1}{d_{\min}\bar{d}_{\min}^{1/2}} + \frac{1}{\bar{d}_{\min}d_{\min}^{1/2}}\right) \left(\frac{1}{d_{\min}^{1/2}} + \frac{1}{\bar{d}_{\min}^{1/2}} 
\right) \|\bar{\BH}\|\|\BD - \bar{\BD}\|  \lesssim \frac{\bar{\sigma}\|\BD - \bar{\BD}\|}{\bar{d}_{\min}^2}.
\end{aligned}
\]
Finally, we have
\begin{align*}
\norm{\BDelta_{\rm sym}} & \lesssim \norm{\BD^{-1/2}\bar{\BH}\BD^{-1/2} - \bar{\BD}^{-1/2}\bar{\BH}\bar{\BD}^{-1/2}} \\
& \lesssim \frac{\bar{\sigma}\|\BD-\bar{\BD}\|}{\bar{d}_{\min}^2} + \frac{M\sqrt{pn\log n}}{\bar{d}_{\min}} \lesssim \frac{M\sqrt{pn \log n} }{\bar{d}_{\min}} \left(  \frac{\bar{\sigma}}{\bar{d}_{\min}} + 1\right) \lesssim \frac{1}{\lambda^2} \sqrt{\frac{\log n}{pn}}
\end{align*} 
where
\[
\frac{M\sqrt{pn \log n} }{\bar{d}_{\min}} \lesssim \frac{1}{\lambda}\sqrt{\frac{\log n}{pn}},~~~~~
\frac{\bar{\sigma}}{\bar{d}_{\min}} \lesssim \frac{\eta p \sqrt{n}\|\br - \alpha\bone_n\|}{\lambda pn M} = \frac{\eta \|\br - \alpha\bone_n\|}{\lambda\sqrt{n}M}  \lesssim \frac{1}{\lambda}.
\]

\subsection{Proof of Theorem \ref{thm:E134_norm}: estimation of $E_1,E_3,$ and $E_4$}\label{ss:proofE134N}

By Lemma~\ref{lem:delta_sym} and Weyl's inequality, 
\begin{equation}\label{eq:xi_pert}
|\xi - \bar{\xi}| \leq \|\BDelta_{\rm sym}\| \lesssim \frac{\bar{\xi}}{\SNR_N} \lesssim \frac{\bar{\xi}}{\lambda^2\SNR}
\end{equation}
where $\SNR_N$ is defined in~\eqref{def:SNRN} and satisfies $\SNR_N\gtrsim \lambda^2\SNR.$

\paragraph{Estimation of $E_1$:} Note that $E_1$ can be bounded by 
\begin{align*}
E_1 & = \frac{1}{\xi} \norm{\BD^{-1}\bar{\BH}(\bpsi_1-\beta\bar{\bpsi}_1)}_\infty = \frac{\|\BD^{-1}\bar{\BD}\|}{\xi} \norm{\mi \bar{\BD}^{-1}\bar{\BH}(\bpsi_1-\beta\bar{\bpsi}_1)}_\infty \\ 
&\lesssim  \frac{\bar{\xi} \|\BD^{-1}\bar{\BD}\|}{\xi}\cdot 
\norm{\lp \bar{\BD}^{-1/2}\bar{\bvphi}_1 \bar{\bvphi}_1^{H}\bar{\BD}^{1/2} - \bar{\BD}^{-1/2}\bar{\bvphi}_2 \bar{\bvphi}_2^{H}\bar{\BD}^{1/2} \rp(\bpsi_1-\beta\bar{\bpsi}_1)}_{\infty} \\ 
&\lesssim \frac{1}{\lambda^{1/2}}\norm{\bar{\bpsi}_1}_{\infty}\norm{\bpsi_1-\beta\bar{\bpsi}_1}
 \end{align*}
where Lemma~\ref{lem:delta_sym} gives $\|\BD^{-1}\bar{\BD}\| = O(1)$, $\kappa(\bar{\BD}) = O(\lambda^{-1})$ and $\bar{\xi}/\xi = O(1)$ under Assumption~\ref{def:SNR2}.
The generalized Davis-Kahan (Theorem \ref{thm:dk}) states that 
\[
\norm{\bpsi_1-\beta\bar{\bpsi}_1} \lesssim \frac{\sqrt{\kappa(\BD)}\norm{\mi \lp \BD^{-1}\BH - \bar{\BD}^{-1}\BH \rp\bar{\bpsi}_1}}{\bar{\xi} - \norm{\BDelta_{\rm sym}}}
\]
where $\bar{\BD}^{-1}\bar{\BH}\bar{\bpsi}_1 = \bar{\xi}\bar{\bpsi}_1$ and the second largest eigenvalue of $\BD^{-1}\BH$ is at most $\|\BDelta_{\rm sym}\|.$
Note that
\[
\label{eqn:normalized dk}
\begin{aligned}
\norm{\BD^{-1}\BH - \bar{\BD}^{-1}\BH} 
& = \|  \BD^{-1}(\BD -\bar{\BD})\bar{\BD}^{-1} \BH\| \lesssim \bar{\sigma} \|\BD^{-1}\|\|\bar{\BD}^{-1}\|\|\BD - \bar{\BD}\| \\
& \lesssim \frac{\bar{\sigma} \|\BD - \bar{\BD}\|}{\bar{d}_{\min}^2} \lesssim \frac{\eta p \sqrt{n}\|\br -\alpha\bone_n\| \cdot M\sqrt{pn\log n}}{(\lambda pnM)^2} \lesssim \frac{1}{\lambda^2}\sqrt{\frac{\log n}{pn}}
\end{aligned}
\]
where $\bar{\sigma}/\bar{d}_{\min} = O(\lambda^{-1})$, $\|\BD- \bar{\BD}\|\lesssim M\sqrt{pn\log n}$ follows from Lemma~\ref{lem:delta_sym}, and $\bar{d}_{\min}\gtrsim \lambda pnM.$ Then
\begin{equation}\label{eq:phi1_dk}
\norm{\bpsi_1-\beta\bar{\bpsi}_1} \lesssim \frac{1}{\bar{\xi}}\cdot \frac{1}{\lambda^{5/2}} \sqrt{\frac{\log n}{pn}} \lesssim\frac{\SNR^{-1}}{\lambda^{5/2}}
\end{equation}
where $\|\BDelta_{\rm sym}\| \lesssim \frac{1}{\lambda^2} \sqrt{\frac{\log n}{pn}}$.
Therefore, $E_1$ is bounded by
\begin{equation}\label{eq:E1_norm}
E_1 \lesssim \lambda^{-3} \SNR^{-1} \|\bar{\bpsi}_1\|_{\infty}.
\end{equation}

\paragraph{Estimation of $E_3$:} Using~\eqref{eq:xi_pert}, $E_3$ can be bounded by
\begin{equation}\label{eq:E3_norm}
E_3 = \frac{|\bar{\xi}-\xi|}{\xi}\norm{\bar{\bpsi}_1}_{\infty} \lesssim \lambda^{-2}\SNR^{-1} \|\bar{\bpsi}_1\|_{\infty}.
\end{equation}

\paragraph{Estimation of $E_4$:} Note that $E_4$  can be decomposed into 
\begin{align*}
E_4 & = \frac{1}{\xi} \norm{\lp \BD^{-1}\BH -\bar{\BD}^{-1}\bar{\BH}  \rp\bar{\bpsi}_1}_\infty  \le \frac{1}{\xi}\norm{\lp \BD^{-1} - \bar{\BD}^{-1} \rp\bar{\BH}\bar{\bpsi}_1}_\infty + \frac{1}{\xi}\norm{\BD^{-1}\BDelta\bar{\bpsi}_1}_\infty \\
& \le\frac{1}{\xi}\norm{\lp \BD^{-1} \bar{\BD}- \I_n \rp \bar{\BD}^{-1}\bar{\BH}\bar{\bpsi}_1}_\infty + \frac{1}{\xi}\norm{\BD^{-1}\BDelta\bar{\bpsi}_1}_\infty \\
& \leq \frac{\bar{\xi} \|\BD - \bar{\BD} \| \|\bar{\bpsi}_1\|_{\infty}}{\xi d_{\min}} + \frac{\|\BDelta\bar{\bpsi}_1\|_{\infty}}{\xi d_{\min}}. 
\end{align*}
where $\mi\bar{\BD}^{-1}\bar{\BH}\bar{\bpsi}_1 = \bar{\xi}\bar{\bpsi}_1.$
Note that $\bar{\xi}/\xi = O(1)$ in~\eqref{eq:xi_pert}, $\|\BD - \bar{\BD}\|\lesssim M\sqrt{pn\log n}$ in Lemma~\ref{lem:delta_sym}, $d_{\min}\gtrsim \lambda pnM$ and
$\|\BDelta\bar{\bpsi}_1\|_{\infty} \lesssim M\|\bar{\bpsi}_1\|_{\infty} \sqrt{pn\log n}$ 
follows from Lemma~\ref{lem:delta}. Then $E_4$ is upper bounded by
\begin{equation}\label{eq:E4_norm}
\begin{aligned}
E_4 & \lesssim \frac{1}{\lambda}\sqrt{\frac{\log n}{pn}} \|\bar{\bpsi}_1\|_{\infty} + \frac{M\|\bar{\bpsi}_1\|_{\infty} \sqrt{pn\log n}}{\xi \lambda pnM} \\
& \lesssim \frac{1}{\lambda}\left( \sqrt{\frac{\log n}{pn}}  + \frac{1}{\xi} \sqrt{ \frac{\log n}{pn}} \right)\|\bar{\bpsi}_1\|_{\infty} \lesssim \frac{\SNR^{-1}}{\lambda}\|\bar{\bpsi}_1\|_{\infty}
\end{aligned}
\end{equation}
which follows from~\eqref{eq:barxi_low} and~\eqref{def:SNR_norm}.

\subsection{Proof of Lemma \ref{thm:E2_norm}: estimation of $E_2$}\label{ss:proofE2N}

As discussed in Section~\ref{ss:tech}, we introduce the auxiliary vector $\bpsi^{(k)}_1$ which is the top eigenvector of $\mi \bar{\BD}^{-1} \BH^{(k)}$ with eigenvalue $\xi^{(k)}$, i.e., $\mi \bar{\BD}^{-1} \BH^{(k)}\bpsi^{(k)}_1 = \xi^{(k)}\bpsi^{(k)}_1$. Then we will decompose $\xi^{-1}d_{\min}^{-1} | \BDelta_k^{\top}(\bpsi_1-\beta\bar{\bpsi}_1)|$ into $T_1$ and $T_2$, and find an upper bound of each one.

\begin{lemma}
\label{lem:Tterm_norm}
Under Assumption~\ref{def:SNR2}, it holds with high probability that
\begin{align*}
T_{1} & = \frac{1}{\xi d_{\min}}  | \BDelta_k^{\top}(\bpsi_1-\beta^{(k)}\bpsi_1^{(k)})|  \lesssim \lambda^{-3}\SNR^{-2} \max_{1\leq k\leq n} \|\bpsi_1^{(k)}\|_{\infty}, \\
T_{2} & = \frac{1}{\xi d_{\min}}  | \BDelta_k^{\top}(\beta^{(k)}\bpsi_1^{(k)}-\beta\bar{\bpsi}_1)| \lesssim \lambda^{-1}\SNR^{-1} \max_{1\leq k\leq n} (\|\bpsi_1^{(k)}\|_{\infty} + \|\bar{\bpsi}_1\|_{\infty}).
\end{align*}
Then $E_2$ satisfies
\begin{equation}
E_2 \leq T_1 + T_2 \lesssim \lambda^{-1}\SNR^{-1} (\|\bpsi_1^{(k)}\|_{\infty} + \|\bar{\bpsi}_1\|_{\infty}) \lesssim \lambda^{-3}\SNR^{-1}(\|\bpsi_1\|_{\infty} + \|\bar{\bpsi}_1\|_{\infty} )
\end{equation}
where $\lambda^2 \SNR\gtrsim 1$ holds under Assumption~\ref{def:SNR2}.
\end{lemma}

\begin{proof}

\noindent{\bf Estimation of $T_1$: } We will bound $T_1$ first: 
\begin{align*}
T_{1} & \leq \frac{1}{\xi d_{\min}} \max_{1\leq k\leq n} | \BDelta_k^{\top}(\bpsi_1-\beta^{(k)}\bpsi_1^{(k)})| \le \frac{\|\BDelta\|}{\xi d_{\min}} \max_{1\leq k\leq n} \| \bpsi_1-\beta^{(k)}\bpsi_1^{(k)}\|.
\end{align*}
Note that $\|\BDelta\|/\xi d_{\min} \lesssim \lambda^{-1}\xi^{-1}\sqrt{\frac{\log n}{pn}} \lesssim \lambda^{-1}\SNR^{-1}$. It remains to estimate $\| \bpsi_1-\beta^{(k)}\bpsi_1^{(k)}\|$:
\[
\| \bpsi_1-\beta^{(k)}\bpsi_1^{(k)}\| \lesssim \frac{\sqrt{\kappa(\BD)}\norm{\lp \BD^{-1}\BH - \bar{\BD}^{-1}\BH^{(k)} \rp \bpsi^{(k)}_1}}{\xi - (\|\BDelta^{(k)}_{\rm sym}\| + \|\BDelta_{\rm sym}\|)} \lesssim
\frac{\norm{\lp \BD^{-1}\BH - \bar{\BD}^{-1}\BH^{(k)} \rp \bpsi^{(k)}_1}}{\lambda\bar{\xi}}  
\]
follows from Theorem~\ref{thm:dk} with $\widehat{\lambda} = \xi^{(k)}$, 
\[
\delta \geq \xi^{(k)} - \|\BDelta_{\rm sym}\| \geq \bar{\xi} - (\|\BDelta^{(k)}_{\rm sym}\| + \|\BDelta_{\rm sym}\|) \geq \bar{\xi} (1- O(\lambda^{-2}\SNR^{-1}))
\] 
and $\BDelta^{(k)}_{\rm sym} = \bar{\BD}^{-1/2}\BDelta^{(k)} \bar{\BD}^{-1/2}$ satisfies
\[
\|\BDelta_{\rm sym}^{(k)}\| \leq \|\BDelta_{\rm sym}\| \lesssim \frac{1}{\lambda^2}\sqrt{\frac{\log n}{pn}},~~\forall 1\leq k\leq n.
\]
Then
\begin{align*}
& \norm{\lp \BD^{-1}\BH - \bar{\BD}^{-1}\BH^{(k)} \rp \bpsi^{(k)}_1} \\
& \le \norm{\lp \BD^{-1} \bar{\BD} - \I_n \rp\bar{\BD}^{-1} \BH^{(k)} \bpsi_1^{(k)}} + \norm{\bar{\BD}^{-1}\lp \BH - \BH^{(k)} \rp\bpsi_1^{(k)}} \\
& \leq \xi^{(k)}  \|\BD^{-1}\bar{\BD} - \I_n\| \|\bpsi^{(k)}_1\|_{\infty} + \frac{\|(\BDelta - \BDelta^{(k)})\bpsi_1^{(k)}\|}{\bar{d}_{\min}} 
\end{align*}
where $\BDelta - \BDelta^{(k)} = \BH - \BH^{(k)}$ and $\mi \bar{\BD}^{-1} \BH^{(k)} \bpsi_1^{(k)} = \xi^{(k)} \bpsi_1^{(k)}$. Here
\begin{align*}
\|\BD^{-1}\bar{\BD} - \I_n\|  \leq \frac{\|\BD - \bar{\BD}\|}{d_{\min}} \lesssim \frac{1}{\lambda}\sqrt{\frac{\log n}{pn}}\lesssim \lambda^{-1}\SNR^{-1} 
\end{align*}
where $\|\BD - \bar{\BD}\|\lesssim M\sqrt{pn\log n}$ follows from Lemma~\ref{lem:delta_sym}.
Note that 
\[
\begin{aligned}
\ls\lp \BDelta - \BDelta^{(k)} \rp\bpsi_1^{(k)} \rs_{\ell} 
= \begin{cases}
\Delta_{\ell k}\psi_{1k}^{(k)},  & \ell\neq k, \\
-\BDelta^\top_{k}\bpsi_1^{(k)}, & \ell = k.
\end{cases}
\end{aligned}
\]
where $\BDelta_k$ and $\bpsi_1^{(k)}$ are independent. Lemma~\ref{lem:delta} implies that
\[
\norm{\lp \BDelta - \BDelta^{(k)} \rp\bpsi_1^{(k)}} \lesssim |\psi_{1k}^{(k)}| \|\BDelta_k\| + |\BDelta_k^{\top}\bpsi_1^{(k)} | \lesssim M\sqrt{pn\log n} \|\bpsi_1^{(k)}\|_{\infty}
\]
where $\BDelta_k$ and $\bpsi_1^{(k)}$ are statistically independent.

Therefore, we have
\[
\norm{\lp \BD^{-1}\BH - \bar{\BD}^{-1}\BH^{(k)} \rp \bpsi^{(k)}_1} 
\lesssim
\left(\xi^{(k)} \lambda^{-1}\SNR^{-1} + \frac{M\sqrt{pn \log n}}{\lambda pnM} \right)\|\bpsi_1^{(k)}\|_{\infty}
\]
where $d_{\min} \gtrsim \lambda pnM.$
This implies
\begin{equation}
\|\bpsi_1 - \beta^{(k)} \bpsi_1^{(k)}\| \leq \frac{1}{\lambda^2}\left(\SNR^{-1} \max_{1\leq k\leq n}\frac{\xi^{(k)}}{\bar{\xi}}  +\frac{1}{\bar{\xi}} \sqrt{\frac{\log n}{pn}}\right) \|\bpsi_1^{(k)}\|_{\infty} \lesssim \lambda^{-2}\SNR^{-1} \max_{1\leq k\leq n} \|\bpsi_1^{(k)}\|_{\infty}.
\end{equation}
Therefore, it holds that
\begin{equation*}
\max_{1\leq k\leq n}\|\bpsi_1 - \beta^{(k)}\bpsi_1^{(k)}\|_{\infty} \lesssim \lambda^{-2}\SNR^{-1}\max_{1\leq k\leq n}\|\bpsi_1^{(k)}\|_{\infty}
\end{equation*}
which implies
\begin{equation}\label{eq:phiphik_norm}
 \max_{1\leq k\leq n}\|\bpsi_1^{k}\|_{\infty} \leq \frac{\|\bpsi_1\|_{\infty}}{1 - O(\lambda^{-2}\SNR^{-1})},~\forall 1\leq k\leq n.
\end{equation}
Finally, it holds under Assumption~\ref{def:SNR2} that
\[
T_1 \lesssim  \lambda^{-1}\SNR^{-1} \max_{1\leq k\leq n} \|\bpsi_1 - \beta^{(k)} \bpsi_1^{(k)}\| \lesssim \lambda^{-3}\SNR^{-2} \|\bpsi_1\|_{\infty}.
\]

\noindent{\bf Estimation of $T_2$: } 
The estimation of $T_2$ follows from Lemma~\ref{lem:delta}:
\begin{align*}
T_2 & \leq \frac{1}{\xi d_{\min}} \max_{1\leq k\leq n}  | \BDelta_k^{\top}(\beta^{(k)}\bpsi_1^{(k)} - \beta\bar{\bpsi}_1)|  \leq \frac{\max_{1\leq k\leq n} \{ |\lag \BDelta_k, \bpsi_1^{(k)}\rag| + |\lag \BDelta_k, \bar{\bpsi}_1\rag|\} }{\xi d_{\min}} \\
& \lesssim \frac{M\sqrt{pn \log n} (\|\bpsi_1^{(k)}\|_{\infty} + \|\bar{\bpsi}_1\|_{\infty} )}{\bar{\xi} \lambda pnM } \lesssim \lambda^{-1}\SNR^{-1} \left(\max_{1\leq k\leq n}\|\bpsi_1^{(k)}\|_{\infty} + \|\bar{\bpsi}_1\|_{\infty}\right)
\end{align*}
where $\xi \asymp \bar{\xi}$ and $d_{\min}\gtrsim \lambda pnM.$ Using~\eqref{eq:phiphik_norm} finishes the proof.
\end{proof}

\section{Matrix perturbation and concentration inequalities}\label{s:C}
In this section, we summarize some results used in the proofs of the main results. 

\begin{theorem}[Weyl's inequality \cite{weyl_asymptotische_1912}]\label{thm:weyl}
	Let $\BA \in \mathbb{C}^{n\times n}$ be Hermitian with eigenvalues $\lambda_1\le \cdots \le \lambda_n$ and $\BB \in \mathbb{C}^{n\times n}$ be Hermitian with eigenvalues $\mu_1\le \cdots \le \mu_n$. Suppose the eigenvalues of $\BA + \BB$ is $\rho_1\le \cdots \le \rho_n$. Then 
	\[
		\lambda_i + \mu_1 \le \rho_i \le \lambda_i + \mu_n.
	\]
\end{theorem}

The following generalized Davis-Kahan theorem{~\cite[Theorem 3]{DLS21}} can deal with the eigenvector perturbation problem for matrices of form $\BN^{-1}\BM$ for some diagonal matrix $\BN$ and Hermitian matrix $\BM$. It is useful in proving the main theorem for the normalized algorithms. When $\BN$ is $\BI_n$, it reduces to the classical Davis-Kahan theorem~\cite{DK70}.
\begin{theorem}[Generalized Davis-Kahan theorem{~\cite[Theorem 3]{DLS21}} ]\label{thm:dk}
Consider the eigenvalue problem $\BN^{-1}\BM \bu = \lambda\bu$ where $\BM$ and $\BN$ are both Hermitian, and $\BN$ is positive definite. Let $\BX$ be the matrix that has the eigenvectors of $\BN^{-1}\BM$ as columns. Then $\BN^{-1}\BM$ is diagonalizable and can be written as
\[
	\BN^{-1}\BM = \BX \BLambda \BX^H = \BX_1 \BLambda_1 \BX_1^H + \BX_2 \BLambda_2 \BX_2^H
\]where
\[
	\BX^{-1} = \begin{bmatrix}
		\BX_1 &\BX_2
	\end{bmatrix}^{-1} = \begin{bmatrix}
		\BY_1^H \\ \BY_2^H
	\end{bmatrix}, ~~~ \BLambda = \begin{bmatrix}
		\BLambda_1 &  \\ &\BLambda_2
	\end{bmatrix}.
\]

Suppose $\delta = \min_{i} |(\BLambda_{2})_{ii} - \widehat{\lambda}|$ is the absolute separation of $\widehat{\lambda}$ from $(\BLambda_{2})_{ii}$, then for any vector $\widehat{\bu}$ we have
\[
	\norm{\BP\widehat{\bu} } \le \frac{\sqrt{\kappa(\BN)}\norm{(\BN^{-1}\BM-\widehat{\lambda} \BI_n)\widehat{\bu}}}{\delta}.
\]where $\BP = (\BY^{\dagger}_2)^H(\BY_2)^H = \BI - (\BX^{\dagger}_1)^H(\BX_1)^H$ is the orthogonal projection matrix onto the orthogonal complement of the column space of $\BX_1$, $\kappa(\BN) = \norm{\BN}\norm{\BN^{-1}}$ is the condition number of $\BN$ and $\BY^{\dagger}_2$ is the Moore-Penrose inverse of $\BY_2$.

When $\BN = \BI$ and $(\hat{\lambda}, \widehat{\bu})$ be an eigen-pair of a matrix $\widehat{\BM}$, we have 
\[
	\sin\theta \leq \frac{\norm{(\BM-\widehat{\BM})\widehat{\bu}}}{\delta}
\] where $\theta$ is the canonical angle between $\widehat{\bu}$ and the column space of $\BX_1$. In this case the theorem reduces to the classical Davis-Kahan theorem \cite{DK70}.
\end{theorem}

\begin{theorem}[Matrix Bernstein~\cite{T12}]
	Consider a finite sequence of independent random matrices $\lc \BZ_k \rc$. Assume that each random matrix satisfies
	\begin{align*}
		\mathbb{E}\BZ_k =0,~~~\norm{\BZ_k} \le R.
	\end{align*}
	Then for all $t \ge 0$,
	\begin{align*}
		\mathbb{P}\lp \norm{\sum_{k}\BZ_k} \ge t\rp \le (d_1 + d_2) \cdot \exp(-\frac{t^2/2}{\sigma^2+ Rt/3}).
	\end{align*}
	where 
	\begin{align*}
		\sigma^2 = \max \lc \norm{\sum_k \mathbb{E}\BZ_k^\top\BZ_k},\norm{\sum_k \mathbb{E}\BZ_k\BZ_k^\top} \rc.
	\end{align*}
	Then with probability at least $1-n^{-\gamma+1}$,
	\begin{align*}
		\norm{\sum_{k}\BZ_k} \le \sqrt{2\gamma \sigma^2\log(d_1+d_2)} + \frac{2\gamma R \log(d_1+d_2)}{3}.
	\end{align*}
\label{thm:bernstein}
\end{theorem}



\end{document}